\documentclass[11pt]{article}
\usepackage[letterpaper, margin = 1in]{geometry} 
\usepackage{natbib}
\usepackage{amsthm}
\usepackage{authblk}
\usepackage{dblfnote} 
\title{Structural Risk Minimization for $C^{1,1}(\R^d)$ Regression}
\author[1]{Adam Gustafson \thanks{adam.marc.gustafson@gmail.com}}
\author[2]{Matthew Hirn \thanks{mhirn@msu.edu}}
\author[3]{Kitty Mohammed \thanks{kittymohammed1985@gmail.com}}
\author[4]{Hariharan Narayanan \thanks{hariharan.narayanan@tifr.res.in}}
\author[5]{Jason Xu \thanks{jqxu@ucla.edu}}
\affil[1, 3]{Department of Statistics, University of Washington}
\affil[2]{Department of Mathematics and Department of Computational Mathematics, Science and Engineering, Michigan State University}
\affil[4]{School of Technology and Computer Science, Tata Institute of Fundamental Research}
\affil[5]{Department of Biomathematics, University of California Los Angeles}

\usepackage{mathtools, amssymb}
\usepackage{xspace, mathabx, verbatim, enumitem, float}
\newtheorem{thm}{Theorem}

\newtheorem{lem}[thm]{Lemma}
\newtheorem{cor}[thm]{Corollary}

\newcommand{\mP}{\ensuremath{\mathcal{P}\xspace}}
\newcommand{\mK}{\ensuremath{\mathcal{K}\xspace}}
\newcommand{\mX}{\ensuremath{\mathcal{X}\xspace}}
\newcommand{\E}{\ensuremath{\mathbb{E}\xspace}}
\newcommand{\R}{\ensuremath{\mathbb{R}\xspace}}
\newcommand{\set}[1]{\ensuremath{\left\{#1\right\}\xspace}}
\newcommand{\Lip}{\text{Lip}}

\DeclareMathOperator*{\arginf}{arg\,inf}

\usepackage{placeins}          
\usepackage{color}
\usepackage{graphicx}
\usepackage{subfigure}
\usepackage{verbatim}
\RequirePackage[colorlinks,citecolor=blue,urlcolor=blue]{hyperref}
\mathtoolsset{showonlyrefs}    

\begin{document}
\maketitle

\begin{abstract}%
One means of fitting functions to high-dimensional data is by providing smoothness constraints.  Recently, the following smooth function approximation problem was proposed by \citet*{herbert2014computing}: given a finite set $E \subset \R^d$ and a function $f: E \rightarrow \R$, interpolate the given information with a function $\widehat{f} \in \dot{C}^{1, 1}(\R^d)$ (the class of first-order differentiable functions with Lipschitz gradients) such that $\widehat{f}(a) = f(a)$ for all $a \in E$, and the value of $\Lip(\nabla \widehat{f})$ is minimal.  An algorithm is provided that constructs such an approximating function $\widehat{f}$ and estimates the optimal Lipschitz constant $\Lip(\nabla \widehat{f})$ in the noiseless setting.

We address statistical aspects of reconstructing the approximating function $\widehat{f}$ from a closely-related class $C^{1, 1}(\R^d)$ given samples from noisy data.  We observe independent and identically distributed samples $y(a) = f(a) + \xi(a)$ for $a \in E$, where $\xi(a)$ is a noise term and the set $E \subset \R^d$ is fixed and known.  We obtain uniform bounds relating the empirical risk and true risk over the class $\mathcal{F}_{\widetilde{M}} = \set{f \in C^{1, 1}(\R^d) \mid \Lip(\nabla f) \leq \widetilde{M}}$, where the quantity $\widetilde{M}$ grows with the number of samples at a rate governed by the metric entropy of the class $C^{1, 1}(\R^d)$.
Finally, we provide an implementation using Vaidya's algorithm, supporting our results via numerical experiments on simulated data. 
\end{abstract}

\section{Introduction}
Regression tasks are prevalent throughout statistical learning theory and machine learning. Given $n$ samples in $E \subset \R^d$ and corresponding values $\mathcal{Y} = \{ y(a) \}_{a \in E} \subset \R$, a regression function $f: \R^d \rightarrow \R$ learns a model for the data $(E, \mathcal{Y})$ that best generalizes to new points $x \notin E$. Absent any prior information on $x$, the best regression function $\widehat{f}$, as measured by the squared loss, is obtained by minimizing the $\ell^2$ empirical risk over a specified function class $\mathcal{F}$,
\begin{equation*}
\widehat{f} = \arginf_{f \in \mathcal{F}} \frac{1}{n} \sum_{a \in E} |f(a) - y(a)|^2,
\end{equation*}
subject to a regularization penalty. If $\mathcal{F}$ is equipped with a norm or semi-norm $\| \cdot \|_{\mathcal{F}}$, then the regularized risk can take either the form
\begin{equation} \label{eqn: regularized risk 1}
\widehat{f} = \arginf_{f \in \mathcal{F}} \frac{1}{n} \sum_{a \in E} | f(a) - y(a)|^2 + \lambda \cdot \Omega ( \| f \|_{\mathcal{F}}),
\end{equation}
or
\begin{equation} \label{eqn: regularized risk 2}
\widehat{f} = \arginf_{f\in \mathcal{F}} \frac{1}{n} \sum_{a \in E} |f(a) - y(a)|^2 \quad \text{subject to} \quad \| f \|_{\mathcal{F}} \leq M,
\end{equation}
where $\lambda$ and $M$ are hyper-parameters, and $\Omega: [0,\infty) \rightarrow \R$ is a monotonically increasing function.

In either case, the quality of $\widehat{f}$ is primarily determined by the functional class $\mathcal{F}$. Recently, numerous state of the art empirical results have been obtained by using neural networks, which generate a functional class $\mathcal{F}$ through the architecture of the network. The class $\mathcal{F}$ is also often taken as the span of a suitably defined dictionary of functions, or a reproducing kernel Hilbert space (RKHS) with an appropriate kernel. For example, when $\Omega (\| f \|_{\mathcal{F}}) = \frac{1}{2} \| f \|_{\mathcal{F}}^2$ and $\mathcal{F}$ is an RKHS, equation \eqref{eqn: regularized risk 1} leads to the popular kernel ridge regression scheme, which has a closed form solution that is simple to compute. 

When $\mathcal{F} = \mathrm{span}(\{ \phi_k \}_k)$, the smoothness of $\widehat{f}$ is determined by the dictionary $\{ \phi_k \}_k$, or if $\mathcal{F}$ is a reproducing kernel Hilbert space, the regularity of $\widehat{f}$ is determined by the kernel. An alternate approach that does not require choice of dictionary or kernel is to specify the smoothness of $\widehat{f}$ directly, by taking $\mathcal{F} = \dot{C}^m(\R^d)$ or $\mathcal{F} = \dot{C}^{m-1,1}(\R^d)$. (The former class contains the functions continuously differentiable up to order $m$. The class of $\dot{C}^{m-1,1}(\R^d)$ functions is similar, but it consists of the functions that are differentiable up to order $m - 1$, with the highest-order derivatives having a finite Lipschitz constant.) However, the computational complexity of minimizing the regularized risk over these spaces is generally prohibitive. An exception is the space $\dot{C}^{0,1}(\R^d)$, which consists of functions $f$ with finite Lipschitz constant, and for which several regression algorithms exist \citep{luxburg:classifyLipFcns2004, beliakov:interpLipFcns, gottlieb:regressMetricLipExt2013, kyng:lipLearnGraphs2015}.

In recent work, \citet*{herbert2014computing} provide an efficient algorithm for computing the interpolant $\widehat{f} \in \dot{C}^{1,1}(\R^d)$ that, given noiseless data $(E, \mathcal{Y})$, minimizes the Lipschitz constant of the gradient. In this paper we extend the methods of \citet{herbert2014computing} to regularized risk optimizations of the form \eqref{eqn: regularized risk 2}. In particular, we consider the noisy scenario in which the function to be reconstructed is not measured precisely on a finite subset, but instead is measured with some uncertainty.

An outline of this paper is as follows. In Section \ref{sec:fip}, we introduce the function interpolation problem considered by \cite{herbert2014computing}, and summarize the solution in the noiseless case in Section \ref{sec:legruyer}.  Next, we consider the setting where the function is measured under uncertainty, and derive uniform sample complexity bounds on our estimator in Section \ref{sec:erm}. The resulting optimization problem can be solved using an algorithm due to \citet*{vaidya1996new}; we provide details on computing the solution to the regularized risk in Sections \ref{sec:vaidya} and \ref{sec:wells}.
We implement the estimator and present reconstruction results on simulated data examples in Section \ref{sec:sim}, supporting our theoretical contributions, and close with a discussion.

\subsection{Noiseless Function Interpolation Problem} \label{sec:fip}

Here we summarize the function approximation problem considered by \cite{herbert2014computing}.  First, recall that the Lipschitz constant of an arbitrary function $g: \R^d \rightarrow \R$ is defined as
\begin{equation}\label{eq:lip}
\Lip(g) := \sup_{x, y \in \R^d,\; x \neq y} \frac{|g(x) - g(y)|}{|x-y|},
\end{equation}
where $|\cdot|$ denotes the standard Euclidean norm, and we note that in the sequel we use this definition on domains other than $\R^d$.  Second, denote the gradient of such an arbitrary $g$ as $\nabla g = (\frac{\partial g}{\partial x_1}, \ldots, \frac{\partial g}{\partial x_d})$.  Finally, let $\dot{C}^{m-1,1}(\R^d)$ be the class of $(m-1)$-times continuously differentiable functions whose derivatives have a finite Lipschitz constant.  In the function approximation problem, we are given a finite set of points $E \subset \R^d$ such that $|E| = n$, and a function $f: E \rightarrow \R$ specified on $E$.  The $\dot{C}^{1, 1}(\R^d)$ function approximation problem as stated by \cite{herbert2014computing} is to compute an interpolating function $\widehat{f}: \R^d \rightarrow \R$ that minimizes
\begin{equation}\label{eq:flip_Rd}
\|f\|_{\dot{C}^{1, 1}(\R^d)} := \inf \set{\Lip(\nabla \widetilde{f}) \mid \widetilde{f}(a) = f(a) \text{ for all } a \in E}.
\end{equation}
%
The question of whether one can even reconstruct such an interpolating function was answered by \citet{whitney1934analytic}.  Presume that we also have access to the gradients of $f$ on $E$, and denote them $\set{D_a f}_{a\in E}$.  In the case of $\dot{C}^{1, 1}(\R^d)$, the polynomials are defined by the specified function and gradient information:
$$
P_a(x) = f(a) + D_a f\cdot (x - a), \quad a \in E, \; x \in \R^d.
$$
Letting $\mP$ denote the space of first-order polynomials (i.e., affine functions), the map $P: \R^d \rightarrow \mP, a \mapsto P_a$ is known as a \emph{1-field}.  For any $f \in \dot{C}^{1, 1}(\R^d)$, the first order Taylor expansions of $f$ are elements of $\mP$, and are known as \emph{jets} \citep*{fefferman2009fitting}, defined as: 
$$
J_a f(x) := f(a) + \nabla f(a)\cdot (x - a), \quad a, x \in \R^d.
$$  
Whitney's Extension Theorem for $\dot{C}^{1, 1}(\R^d)$ may be stated as follows:
\begin{thm}[Whitney's Extension Theorem for $\dot{C}^{1, 1}(\R^d)$]
  Let $E \subset \R^d$ be closed and let $P: E \rightarrow \mP$ be a $1$-field with domain $E$.  If there exists a constant $M < \infty$ such that 
  \begin{enumerate}
    \item[($W_0$)] $|P_a(a) - P_b(a)| \leq M |a - b|^2$ for all $a, b \in E$, and
    \item[($W_1$)] $|\frac{\partial P_a}{\partial x_i}(a) - \frac{\partial P_b}{\partial x_i}(a)| \leq M|a-b|$ for all $a, b \in E$, and $i \in \set{1, \ldots, d}$,
  \end{enumerate}
  then there exists an extension $\widehat{f} \in \dot{C}^{1, 1}(\R^d)$ such that $J_a \widehat{f} = P_a$ for all $a \in E$.  
\end{thm} 
\noindent Given a finite set $E$ as in the function interpolation problem, these conditions are automatically satisfied.  However, this theorem does not provide a solution for the minimal Lipschitz constant of $\nabla f$. \citet*{le2009minimal} provides a solution to both problems, which we discuss next.

\subsection{Minimal Value of \boldmath{$\Lip(\nabla \widehat{f})$}} \label{sec:legruyer}
\cite{herbert2014computing} define the following norm for when the first-order polynomials are known
\begin{equation}\label{eq:plip}
\|P\|_{\dot{C}^{1, 1}(E)} := \inf \set{\Lip(\nabla \widetilde{f}) \mid J_a \widetilde{f} = P_a \text{ for all } a \in E},
\end{equation}
and similarly define
\begin{equation}\label{eq:flip}
\|f\|_{\dot{C}^{1, 1}(E)} := \inf\set{\Lip(\nabla \widetilde{f}) \mid\widetilde{f}(a) = f(a) \text{ for all } a \in E}, 
\end{equation}
when the gradients $\set{D_a f}_{a \in E}$ are unknown, where in both cases the infimum is taken over functions $\widetilde{f}: E \rightarrow \R$.  

Presuming we are given the $1$-field $P: E \rightarrow \mP, \, a \mapsto P_a$, \cite{le2009minimal} defines the functional $\Gamma^1$ as:
\begin{equation} \label{eq:legruyerfunc}
\Gamma^1(P; E) = 2 \sup_{x \in \R^d} \left(\max_{a, b \in E, \; a \neq b} \frac{P_a(x) - P_b(x)}{|a - x|^2 + |b - x|^2}\right).
\end{equation}
Given only functions $f: E \rightarrow \R$, \cite{le2009minimal} also defines the functional $\Gamma^1$ in terms of $f$ as
\begin{equation}\label{eq:func}
\Gamma^1(f; E) = \inf\set{\Gamma^1(P; E) \mid P_a(a) = f(a) \text{ for all } a \in E},
\end{equation}
The following theorem is proven by \cite{le2009minimal}, which shows that \eqref{eq:legruyerfunc} and its equivalent formulation in \eqref{eq:pfunc} provides a solution for \eqref{eq:plip}:
\begin{thm}[Le Gruyer]\label{thm:legruyer}
Given a set $E \subset \R^d$ and a $1$-field $P: E \rightarrow \mP$, 
$$
\Gamma^1(P; E) = \|P\|_{\dot{C}^{1, 1}(E)}.
$$
\end{thm} \noindent
An equivalent formulation of \eqref{eq:legruyerfunc} which is amenable to implementation is as follows.  Consider the following functionals mapping $E \times E \rightarrow [0, \infty]$:
\begin{align*}
A(P; a, b) &=  \frac{(P_a(a) - P_b(a)) + (P_a(b) - P_b(b))}{|a-b|^2} \\
&= \frac{2(f(a)-f(b)) + (D_a f - D_b f)\cdot (b-a)}{|a-b|^2}, \\
B(P; a, b) &=  \frac{|\nabla P_a(a) - \nabla P_b(a)|}{|a-b|} \\
&= \frac{|D_a f - D_b f|}{|a-b|}.
\end{align*}
Proposition 2.2 of \cite{le2009minimal} states that
\begin{equation}\label{eq:pfunc}
\Gamma^1(P; E) = \max_{a \neq b \in E} \sqrt{A(P; a, b)^2 + B(P; a, b)^2} + |A(P; a, b)|,
\end{equation}
whence a naive implementation allows $\Gamma^1(P; E)$ to be found in $O(n^2)$ computations.  Inspired by \cite{fefferman2009fitting}, \cite{herbert2014computing} also construct algorithms which will solve for the order of magnitude of $\|P\|_{\dot{C}^{1, 1}(E)}$ in $O(n \log n)$ time, but we omit the details here.  Additionally, as a consequence of the proof of Proposition 2.2 of \cite{le2009minimal}, equation \eqref{eq:legruyerfunc} may alternatively be written as
\begin{equation}\label{eq:legruyerfunc2}
\Gamma^1(P; E) = 2 \max_{a, b \in E:\, a \neq b} \sup_{x \in \bar{B}^d \left(\frac{a+b}{2}, \frac{|a-b|}{2}\right)}\frac{P_a(x) - P_b(x)}{|a-x|^2 + |b-x|^2},
\end{equation}
where $\bar{B}^d(z, r)$ denotes the closed $d$-dimensional Euclidean ball centered at $z$ with radius $r$.

Recall that the gradients $\set{D_a f}_{a \in E}$ are typically not known in applications.  As a corollary, we have the following convex optimization problem for finding \eqref{eq:flip}, and the minimizing $1$-field provides the gradients $\set{D_a f}_{a \in E}$.  
\begin{cor}
  Given a set $E \subset \R^d$ and a function $f: E \rightarrow \R$, 
  $$
  \Gamma^1(f; E) = \|f\|_{\dot{C}^{1, 1}(E)}.
  $$
\end{cor} \noindent
Recall that $P_a(x) = f(a)  + D_a f\cdot (x-a)$.  The set $E \subset \R^d$ and the values $\set{f(a)}_{a \in E}$ are fixed, so the optimization problem is to solve for the gradients $\set{D_a f}_{a \in E}$ that minimize $\Gamma^1(P; E)$.  

\section{\boldmath{$\dot{C}^{1, 1}(E)$} Regression} \label{sec:c11_reg}
In statistical applications where $f(a)$ is observed with uncertainty, one often assumes that we observe $\set{y(a)}_{a \in E}$, where $y(a) = f(a) + \xi(a)$, and $\xi(a)$ is assumed to be independent and identically distributed Gaussian noise for each $a \in E$.  Since both the function values and the gradients $\set{f(a), D_a f}_{a \in E}$ are unknown, we minimize an empirical squared error loss over the $k := (d+1)n$ variables defining the $1$-field.  Given a bound on the $\dot{C}^{1,1}(E)$ seminorm of the unknown $1$-field, regression entails solving an optimization problem of the form
\begin{equation}\label{eq:alg}
\begin{aligned}
\min_{P} \quad & \frac{1}{n} \sum_{a \in E} (y(a) - P_a(a))^2   \\
\text{s.t.} \quad &\|P\|_{\dot{C}^{1, 1}(E)} \leq M. 
\end{aligned}
\end{equation}
This is a convex optimization problem: the objective function of the empirical squared error loss in \eqref{eq:alg} is convex, as is the constraint set since it is a ball specified by a seminorm. This section proceeds as follows: we begin by analyzing the sample complexity of the function class. These risk bounds establish almost sure convergence of the empirical risk minimizer, and guides the choice of $M$. Given $M$, we next appeal to Vaidya's algorithm to solve the resulting optimization problem \eqref{eq:alg}. We then apply the efficient algorithm of \cite{herbert2014computing} to compute the optimal interpolating function.

\subsection{Sample Complexity and Empirical Risk Minimization}\label{sec:erm}
The constant $M > 0$ will be chosen via sample complexity arguments. To this end, we derive uniform risk bounds for classes of continuous functions $f: B^d \rightarrow \R$, where $B^d$ denotes the unit Euclidean ball in $\R^d$. The function classes of interest are defined in terms of $C^{1, 1}$-norm balls as
$$
\mathcal{F}_{\widetilde{M}} = \set{f \mid \|f\|_{C^{1, 1}(B^d)} \leq \widetilde{M}},
$$
where we are using the norm
\begin{equation}\label{eq:c11_norm}
\|f\|_{C^{1, 1}(B^d)} := \max \set{ \sup_{x \in B^d}\left| f \right|,\,  \sup_{x \in B^d}\left| \nabla{f} \right|,\, \Lip(\nabla{f})}.
\end{equation}
We note that in order to derive the uniform risk bounds for the function classes $\mathcal{F}_{\widetilde{M}}$, we require such classes be compact, which necessitates the choice of the $C^{1, 1}(B^d)$ norm in equation \eqref{eq:c11_norm} as opposed to the $\dot{C}^{1, 1}(B^d)$ seminorm.  

With some abuse of notation, in this section we let $f^*$ denote the underlying function from which we observe noisy samples.  We observe an i.i.d. sample
$$
S = \set{(x_1, y_1),\ldots, (x_n, y_n)}
$$
drawn from a probability distribution $\mathcal{P} = S_{X} \bigtimes S_{Y|X}$ supported on $X \bigtimes Y \subset B^d \bigtimes \R$ under the assumption
\begin{align*}
S_{Y|X} &\sim \mathcal{N}(f^*(X), \sigma^2), \quad \text{where } \|f^*\|_{C^{1, 1}(B^d)} = M^*.
\end{align*}
Since we are in the regression setting, we use squared error loss
\begin{align*}
L(f(x), y) = \left(f(x) - y\right)^2.
\end{align*}
The true risk is defined as the expectation of $L$ over $\mathcal{P}$:
\begin{align*}
R(f) = \mathbb{E}_\mathcal{P}\left[L(f(x), y)\right],
\end{align*}
and the empirical risk is the expectation over $S_S$, the empirical distribution on the sample $S$:
\begin{align*}
\widehat{R}(f) = \mathbb{E}_{S_S}\left[L(f(x), y)\right].
\end{align*}

In order for the empirical risk minimization procedure to converge to a minimizer of the true risk, we need to bound
\begin{align*}
\sup_{f \in \mathcal{F}_{\widetilde{M}}} \left| \widehat{R}(f) - R(f)\right|
\end{align*}
with high probability. The most natural way to do so is by expanding the risk and appealing to entropy methods (i.e., covering number bounds) and standard concentration results. Recall that the covering number $N(\eta, \mathcal{G}, \|\cdot\|)$ is the minimum number of norm balls of radius $\eta$ needed to cover a function class $\mathcal{G}$. We briefly discuss how this is useful toward deriving uniform bounds.

Given a class $\mathcal{G}$ of bounded functions $g: B^d \rightarrow \R$, an i.i.d. sample $T =~\set{z_1, \ldots z_n}$ drawn from a probability distribution $\mathcal{Q}$ supported on $B^d$, and a vector of i.i.d. Rademacher random variables  $\sigma =~\left(\sigma_1, \ldots ,\sigma_n \right)$, the following holds for $0 < \delta < 1$:
\begin{align*}
\mathbb{P} \left[\sup_{g \in \mathcal{G}} \left| \mathbb{E}_{\mathcal{Q}_T}g - \mathbb{E}_{\mathcal{Q}}g\right| < 2 \mathcal{R}_n(\mathcal{G}) + \sqrt{\frac{2 \log(2/\delta)}{n}}\right] > 1 - \delta,
\end{align*}
where $\mathcal{Q}_T$ is the empirical distribution on $T$ and
\begin{align*}
\mathcal{R}_n(\mathcal{G}) = \mathbb{E}_\sigma \frac{1}{n} \left[\sup_{g \in \mathcal{G}}\left(\sum_{i = 1}^{n} \sigma_i g(x_i)\right)\right]
\end{align*} 
is the Rademacher average conditional on the sample. \citet*{refineddudley} show that
\begin{align*}
	\mathcal{R}_n(\mathcal{G}) \leq \inf_{\gamma \geq 0} \left\{4 \gamma + 12 \int_{\gamma}^{\sup_{g \in \mathcal{G}} \|g\|_{\infty}} \sqrt{\frac{\log N(\eta, \mathcal{G}, \|\cdot\|_{\mathcal{L}_2(\mathcal{Q}_T)})}{n}} d\eta \right\},
\end{align*}
where the right-hand side is a modified version of Dudley's entropy integral.

Since we are interested in bounding the risk, we use Lemma \ref{lem:bartlettmendelson} to relate the Rademacher complexity of the loss class and the original class. We provide a proof based on three results due to \citet*{bartlett2003rademacher}; these require familiarity with McDiarmid's inequality, stated below in Lemma \ref{lem:mcdiarmid}.  

\begin{lem}[McDiarmid's inequality, \citealp*{mcdiarmid1989method}]
\label{lem:mcdiarmid}
Let $X_1, \dots, X_n$ be independent random variables that take values in a set $A$. Suppose the function $f: A^n \rightarrow \R$ satisfies
\begin{align*}
\sup_{x_1, \dots, x_n, x'_i \in A} \left|f(x_1, \dots, x_n) - f(x_1, \dots, x_{i-1}, x'_i, x_{i + 1}, \dots, x_n)\right| \leq c_i
\end{align*}
for every $1 \leq i \leq n$. Then, for $t > 0$,
\begin{align*}
\mathbb{P}\left[\left|f(X_1, \dots, X_n) - \E f(X_1, \dots, X_n)\right| \geq t\right] \leq 2e^{-2t^2/\sum_{i = 1}^{n} c_i^2}.
\end{align*}
\end{lem}

\begin{lem}
\label{lem:bartlettmendelson}
	Let $\mathcal{F}_{\widetilde{{M}}}$ be a class of functions $f: B^d \rightarrow \R$ with $\sup_{f \in \mathcal{F}} \left|f\right| \leq \widetilde{{M}}$. Let $L: [-\widetilde{{M}}, \widetilde{{M}}] \bigtimes Y \rightarrow \R$ be a bounded loss function with Lipschitz constant $L_L$ and $0 \leq L \leq L_{max}$. Then, the following is true for $0 < \delta < 1$:
	\begin{align*}
	\mathbb{P} \left[\sup_{f \in \mathcal{F}_{\widetilde{{M}}}} \left|R(f) - \widehat{R}(f)\right| < 4L_L \mathcal{R}_{n}(\mathcal{F}_{\widetilde{{M}}}) + 7L_{max}\sqrt{\frac{\log(8/\delta)}{2n}} \right] > 1 - \delta.
	\end{align*} 
\end{lem}
\begin{proof}
	In this lemma, we begin by adapting Theorem 8 of \cite{bartlett2003rademacher} to find a bound on the risk that depends on a probabilistic term plus the expectation of the Rademacher average of the class of loss functions. We follow the proof of Lemma 4 of \cite{bousquet2004introduction} for guidance. We apply the two-sided form of McDiarmid's inequality as we want bounds on the absolute value of $R(f) - \widehat{R}(f)$, and appeal to Theorems 11 and 12 of \cite{bartlett2003rademacher} to relate the expected Rademacher average of the loss class to the empirical Rademacher average of $\mathcal{F}_{\widetilde{{M}}}$. 
		
	Let $\widetilde{L} \circ \mathcal{F}_{\widetilde{{M}}}$ be the class of functions consisting of $\set{\left(x, y\right) \mapsto L(f(x), y) - L(0, y)}$. If $h \in \widetilde{L} \circ \mathcal{F}_{\widetilde{{M}}}$, then $-L_{max} \leq h \leq L_{max}$. For any $f \in \mathcal{F}_{\widetilde{{M}}}$, the triangle inequality shows that
	\begin{align*}
	\left|R(f) - \widehat{R}(f)\right| \leq \sup_{h \in \widetilde{L} \circ \mathcal{F}_{\widetilde{{M}}}} \left|\E h - \widehat{\E}_n h\right| + \left|\E L(0, y) - \widehat{\E}_n L(0, y) \right|.
	\end{align*}
	McDiarmid's inequality yields more favorable expressions for both terms on the right-hand side as follows. The most that $\widehat{\E}_n L(0, y)$ can change by altering one sample is $L_{max}/n$. Since $\E \widehat{\E}_n L(0, y) = {\E} L(0, y)$, we have, with probability $1 - \delta/4$,
	\begin{align*}
	\left|\E L(0, y) - \widehat{\E}_n L(0, y) \right| \leq \sqrt{\frac{L_{max}^2 \log 8/\delta}{2n}}.
	\end{align*}
	The most that $\sup_{h \in \widetilde{L} \circ \mathcal{F}_{\widetilde{{M}}}} \left|\E h - \widehat{\E}_n h\right|$ can change with an alteration of one sample is $2 L_{max}/n$. Therefore, with probability $1 - \delta/4$,
	\begin{align*}
	\left|\sup_{h \in \widetilde{L} \circ \mathcal{F}_{\widetilde{{M}}}} \left|\E h - \widehat{\E}_n h\right| - \E \sup_{h \in \widetilde{L} \circ \mathcal{F}_{\widetilde{{M}}}} \left|\E h - \widehat{\E}_n h\right|\right| \leq \sqrt{\frac{4 L_{max}^2 \log 8/\delta}{2n}}.
	\end{align*}
	Now,
	\begin{align*}
	\E \sup_{h \in \widetilde{L} \circ \mathcal{F}_{\widetilde{{M}}}} \left|\E h - \widehat{\E}_n h\right| \leq \max\set{\E \sup_{h \in \widetilde{L} \circ \mathcal{F}_{\widetilde{{M}}}} \left(\E h - \widehat{\E}_n h \right), \E \sup_{h \in \widetilde{L} \circ \mathcal{F}_{\widetilde{{M}}}} \left(\widehat{\E}_n h - \E h\right)}.
	\end{align*}
	Let $S' := \set{(x'_1, y'_1),\ldots, (x'_n, y'_n)}$ be a sample with the same distribution as $S$. Conditioning on the original sample,
	\begin{align*}
	\E \sup_{h \in \widetilde{L} \circ \mathcal{F}_{\widetilde{{M}}}} \left(\E h - \widehat{\E}_n h \right) & = \E \sup_{h \in \widetilde{L} \circ \mathcal{F}_{\widetilde{{M}}}} \E \left[\frac{1}{n} \sum_{i = 1}^{n} h(x'_i, y'_i) - \widehat{\E}_n h \bigg | S \right]\\
	& \leq \E \sup_{h \in \widetilde{L} \circ \mathcal{F}_{\widetilde{{M}}}} \left(\frac{1}{n} \sum_{i = 1}^{n} h(x'_i, y'_i) - \widehat{\E}_n h \right)\\
	& = \E \sup_{h \in \widetilde{L} \circ \mathcal{F}_{\widetilde{{M}}}}  \frac{1}{n} \sum_{i = 1}^{n} \sigma_i \left(h(x'_i, y'_i) - h(x_i, y_i)\right)\\
	& \leq \E \sup_{h \in \widetilde{L} \circ \mathcal{F}_{\widetilde{{M}}}}  \frac{1}{n} \sum_{i = 1}^{n} \sigma_i h(x'_i, y'_i) + \E \sup_{h \in \widetilde{L} \circ \mathcal{F}_{\widetilde{{M}}}}  \frac{1}{n} \sum_{i = 1}^{n} -\sigma_i h(x_i, y_i)\\
	& = 2 \E \mathcal{R}_n(\widetilde{L} \circ \mathcal{F}_{\widetilde{{M}}}).
	\end{align*}
	The second line follows by applying Jensen's inequality to $\sup$, which is convex. Note the preceding argument is symmetric in $\E h$ and $\widehat{\E}_n h$. Therefore, $\E \sup_{h \in \widetilde{L} \circ \mathcal{F}_{\widetilde{{M}}}} \left|\E h - \widehat{\E}_n h\right|$ has the same upper bound and, with probability $1 - \delta/2$,
	\begin{align*}
	\left|R(f) - \widehat{R}(f)\right| \leq 2 \E \mathcal{R}_n(\widetilde{L} \circ \mathcal{F}_{\widetilde{{M}}}) + 3 L_{max} \sqrt{\frac{\log(8/\delta)}{2 n}}.
	\end{align*}
	
Theorem 11 of \cite{bartlett2003rademacher} uses McDiarmid's inequality to bound the difference between the empirical and expected Rademacher averages, but assumes that we are interested in the Rademacher complexity of a class of functions mapping to $[-1, 1]$. Since $\mathcal{F}_{\widetilde{{M}}}$ maps to $[-\widetilde{M}, \widetilde{M}]$, we rederive the analogous result here. The most that one sample affects $\mathcal{R}_n(\widetilde{L} \circ \mathcal{F}_{\widetilde{{M}}})$ is $2 L_{max}/n$. We have
	\begin{align*}
	\mathbb{P}\left[\left|\mathcal{R}_n(\widetilde{L} \circ \mathcal{F}_{\widetilde{{M}}}) - \E \mathcal{R}_n(\widetilde{L} \circ \mathcal{F}_{\widetilde{{M}}})\right| \geq t\right] \leq 2e^{-2nt^2/(4 L_{max}^2)}.
	\end{align*} 
	Thus, with probability $1 - \delta/2$,
	\begin{align*}
	2 \E \mathcal{R}_n(\widetilde{L} \circ \mathcal{F}_{\widetilde{{M}}}) & \leq 2 \mathcal{R}_n(\widetilde{L} \circ \mathcal{F}_{\widetilde{{M}}}) + 4 L_{max} \sqrt{\frac{\log(4/\delta)}{2n}}\\
	& \leq 4 L_L \mathcal{R}_n(\mathcal{F}_{\widetilde{{M}}}) + 4 L_{max} \sqrt{\frac{\log(8/\delta)}{2n}}.
	\end{align*} 
	
	The second line follows from part 4 of Theorem 12, which states that, for $\widetilde{L}: \R \rightarrow \R$ with Lipschitz constant $L_{\widetilde{L}}$ and satisfying $\widetilde{L}(0) = 0$, $\E \mathcal{R}_n(\widetilde{L} \circ \mathcal{F}_{\widetilde{{M}}}) \leq 2 L_{\widetilde{L}} \E \mathcal{R}_n( \mathcal{F}_{\widetilde{{M}}})$. The reasoning from the proof also applies to the empirical Rademacher average, giving $\mathcal{R}_n(\widetilde{L} \circ \mathcal{F}_{\widetilde{{M}}}) \leq 2 L_{\widetilde{L}} \mathcal{R}_n( \mathcal{F}_{\widetilde{{M}}})$. Since $\widetilde{L}$ has the same Lipschitz constant as $L$, we use the notation $L_L$. 
	
	Finally, with probability at least $1 - \delta$,
	\begin{align*}
	\left|R(f) - \widehat{R}(f)\right| \leq 4 L_L \mathcal{R}_n(\mathcal{F}_{\widetilde{{M}}}) + 7 L_{max} \sqrt{\frac{\log(8/\delta)}{2n}}.
	\end{align*}

\end{proof}

Next, the following lemma gives an upper bound on the covering number $N(\eta, \mathcal{F}_{\widetilde{M}}, \|\cdot\|_\infty)$ of $\mathcal{F}_{\widetilde{M}}$ with respect to the supremum norm.
\begin{lem}[adapted from Theorem 2.7.1 of \citealp*{van1996weak}]
\label{lem:vdvw 2.7.1}
	There exists a constant $K$ depending only on $d$ such that, for every $\eta > 0$,
	$$
	\log N(\eta, \mathcal{F}_{\widetilde{M}}, \|\cdot\|_\infty) \leq K \left(\frac{\widetilde{M}}{\eta}\right)^{d/2}.
	$$
\end{lem}

\begin{proof}
Following \cite{van1996weak}, every $f \in \mathcal{F}_{\widetilde{M}}$ is continuous on the open unit ball $B^d$ by assumption, so Taylor's theorem applies everywhere. Fix $\delta = \varepsilon^{1/2} \leq 1$, and take the $\delta$-net of points $x_1, \ldots, x_m$ in $B^d$, where the number of points $m$ is less than or equal to the volume of $B^d$ times a constant that only depends on $d$. Then for all vectors $k$ whose sum of entries $\overline{k}:= \sum_i k_i$ do not exceed $1$ (this includes the zero vector and standard basis vectors), define for each $f$ the vector 
$$A_k f = \left( \Bigl\lfloor \frac{D^k f(x_1)}{\delta^{2-\overline{k}}} \Bigr\rfloor, \ldots, \Bigl\lfloor \frac{D^k f(x_k)}{\delta^{2-\overline{k}}} \Bigr\rfloor , \ldots, \Bigl\lfloor \frac{D^k f(x_m)}{\delta^{2-\overline{k}}} \Bigr\rfloor \right) .$$
The vector $\delta^{2-\overline{k}} A_k f $ thus consists of the $D^k f(x_i)$ values discretized on a mesh with grid width $\delta^{2-\overline{k}}$.

Now, if $f, g \in \mathcal{F}_{\widetilde{M}}$ are such that $A_k f = A_k g$ for all $k$, then $\|f-g\|_\infty \leq C \varepsilon$ for a constant $C$, implying that for each $x$ there exists an $x_i$ such that $\| x - x_i \| \leq \delta$. The remainder term in the Taylor expansion 
$$(f-g)(x) = \sum_{k \leq 1} D^k (f-g)(x_i) \frac{(x-x_i)^k}{k!} + R$$
is bounded by the mesh width $\delta$: indeed, we may consider an integral form via the Fundamental Theorem of Calculus:
\begin{align*}
(f-g)(x) &= (f-g)(x_i) + \int_{x_i}^x D (f-g)( s) \,\mathrm{d} s \\
&= (f-g)(x_i) + \int_{x_i}^x D(f-g)(x_i) + D(f-g)(s) - D(f-g)(x_i) \,\mathrm{d} s \\
&= (f-g)(x_i) + D(f-g)(x-x_i) + \int_{x_i}^x D(f-g)(s) - D(f-g)(x_i) \,\mathrm{d}s \\
& \leq  (f-g)(x_i) + D(f-g)(x-x_i) + C \| x-x_i \|^2 
\end{align*}
where the last line follows from $f-g \in \mathcal{F}_{\widetilde{M}}$. Therefore we see that the remainder term $\lvert R \rvert \propto \| x-x_i\|^2$, and we may next substitute the mesh width bounding this quantity:
$$ |f-g|(x) \propto \sum_{\overline{k} \leq 1} \delta^{2- \overline{k}} \prod_{i=1}^d \frac{ \delta_i^{\overline{k}} }{k_i!} + \delta^2 \leq  \delta^2(e^d+1).$$
Thus, there exists $C = C(d)$ such that the covering number $N(C_\varepsilon,\mathcal{F}_{\widetilde{M}},\| \cdot \|_\infty)$ is bounded by the number of different matrices $\left\{Af\right\}$ whose rows are the vectors $A_k f$ for $k$ such that $\overline{k} \leq 1$ and $f$ ranges over $\mathcal{F}_{\widetilde{M}}$. There are $d+1$ such vectors. 
Now, by definition of $A_k f$ and using $| D^k f(x_i) | \leq \widetilde{M}$ for all $i$ the number of values of each element in each row is at most $2\widetilde{M}/\delta^{2-\overline{k}} + 1 \leq 2 \widetilde{M}\delta^{-2} + 1$. Thus, each column of $\left\{Af \right\}$ has at most $(2 \widetilde{M} \delta^{-2} + 1)^{d+1}$ values. Note that this already suffices to produce a finite bound. Following \cite{van1996weak} and applying Taylor's theorem again yields a less crude bound $\# \left\{Af \right\} \leq (2 \delta^{-2} + 1)^{d+1} C^{m-1}$ where $C$ is a constant depending only on $d$. We may replace $\delta$ in this expression by $\varepsilon^{1/2}$ and $m$ by its upper bound $\text{Vol}_d(B^d)\varepsilon^{-d/2}$, where $\text{Vol}_d(\cdot)$ represents the $d$-dimensional volume. Now, 
the lemma follows by taking logarithms, bounding $\log(1/\varepsilon)$ by $K(1/\varepsilon)^{d/2}$, and combining all constant terms into $K$.
\end{proof}

We are now ready to provide risk bounds in the following theorem.
\begin{thm}
\label{thm:maintheoremERM}
	Suppose we set  $\widetilde{M} := n ^{1/{(2\widetilde{d})}}$, where $\widetilde{d} := \max \set{d, 5}$, and let $\mathcal{F}_{\widetilde{M}}$ be the class of functions with $C^{1, 1}$-norm bounded above by $\widetilde{M}$.
	\begin{enumerate}[label=(\roman*)]
	\item 
	For $0 < \delta < 1$,
	\begin{align*}
	\mathbb{P} \left[\sup_{f \in \mathcal{F}_{\widetilde{{M}}}} \left|R(f) - \widehat{R}(f)\right| < \varepsilon \right] > (1 - \delta)(1 - e^{-n^{2/\max \set {d, 5}}/{2 \sigma^2}}),
	\end{align*}
	where $\varepsilon$ is a monotonically-decreasing function of $n$ for large enough $n$ and $\lim\limits_{n \rightarrow \infty} \varepsilon = 0$.
	
	\item \begin{align*}
	\sup_{f \in \mathcal{F}_{\widetilde{{M}}}} \left|R(f) - \widehat{R}(f)\right| \xrightarrow{a. s.} 0.
	\end{align*}
	\end{enumerate}

\end{thm}

\begin{proof}
	$\mathcal{F}_{\widetilde{{M}}}$ is a sequence of function classes with increasing $C^{1, 1}$-norm. We set the rate $\widetilde{M} := n ^{1/{(2\widetilde{d})}}$, where $\widetilde{d} := \max \set{d, 5}$, so that $f^*$ is a candidate for large enough $n$. We aim to use Lemma \ref{lem:bartlettmendelson} to prove the desired probability statement, but our loss function is unbounded since $Y$ can be arbitrarily large. To circumvent this, we also let the maximum value of $\set{y_i}$ increase with $n$; samples violating this condition are part of the error probability. Write $y = f^{*}(x) + \xi$, where $\xi \sim \mathcal{N}(0, \sigma^2)$. We condition on the event $\mathcal{H} := \set{\max_{1 \leq i \leq n} \left|\xi_i\right| \leq \sigma \sqrt{2 \log 2n} + n^{1/\widetilde{d}}}$. Theorem 7.1 of \cite{ledoux2005concentration} gives the following bound for suprema of Gaussian processes:
	\begin{align*}
	\mathbb{P} \left[\max_{1 \leq i \leq n} \left|\xi_i\right| < \mathbb{E}{\max_{1 \leq i \leq n} \left|\xi_i\right|} + r \right] > 1 - e^{-r^2/{2 \sigma^2}}.
	\end{align*}
	The following is well-known \citep*{boucheron2013concentration}:
	\begin{align*}
	\mathbb{E}{\max_{1 \leq i \leq n} \left|\xi_i\right|} \leq \sigma \sqrt{2 \log 2n}.
	\end{align*}
	Thus, $\mathbb{P}\left(\mathcal{H}\right) > 1 - e^{-n^{2/\widetilde{d}}/{2 \sigma^2}}$.\\ 
	
	Since the loss function is bounded after conditioning on $\mathcal{H}$, we can compute:
	\begin{align*}
	L_{max} & < \sup_{x, y, f} \left(f(x) - y\right)^2\\
	& < \sup_{x, y, f} \left(\left|f(x)\right| + \left|y\right|\right)^2 \\
	& < \left(\widetilde{{M}} + M^{*} + \sigma \sqrt{2 \log 2n} + n^{1/\widetilde{d}}\right)^2\\
	& = \left(n ^{1/{(2\widetilde{d})}} + M^{*} + \sigma \sqrt{2 \log 2n} + n^{1/\widetilde{d}}\right)^2\\
	& := \widetilde{L}_{max}.
	\end{align*}
	We also find the Lipschitz constant as follows, where $f_1, f_2 \in \mathcal{F}_{\widetilde{M}}$:
	\begin{align*}
	\sup_{x, y, f_1, f_2} \left|\left(f_1(x) - y\right)^2 - \left(f_2(x) - y\right)^2\right| & = \sup_{x, y, f_1, f_2} \left|\left(-2y + f_1(x) + f_2(x)\right)\left(f_1(x) - f_2(x)\right)\right|\\
	& \leq \sup_{x, y, f_1, f_2} \left|-2y + f_1(x) + f_2(x)\right| \|f_1 - f_2\|_{\infty}.
	\end{align*}
	This implies:
	\begin{align*}
	L_{L} & \leq \sup_{x, f_1, f_2} \left|f_1(x) + f_2(x)\right| + 2 \sup_{y} \left|y\right|\\
	& < 2\left(\widetilde{{M}} + M^{*} + \sigma \sqrt{2 \log 2n} + n^{1/\widetilde{d}}\right)\\
	& = 2\left(n ^{1/{(2\widetilde{d})}} + M^{*} + \sigma \sqrt{2 \log 2n} + n^{1/\widetilde{d}}\right)\\
	& := \widetilde{L}_{L}.
	\end{align*}
		
	Next, we bound the Rademacher complexity using the entropy integral:
	\begin{align*}
	\mathcal{R}_n({\mathcal{F}_{\widetilde{{M}}}}) &\leq \inf_{\gamma \geq 0} \left\{4 \gamma + 12 \int_{\gamma}^{\widetilde{{M}}} \sqrt{\frac{\log N(\eta, \mathcal{F}_{\widetilde{{M}}}, 
	\|\cdot\|_{\mathcal{L}_2(S_S)})}{n}} d\eta \right\}\\
	& \leq \inf_{\gamma \geq 0} \left\{4 \gamma + 12 \int_{\gamma}^{\widetilde{{M}}} 
	\sqrt{\frac{\log N(\eta, \mathcal{F}_{\widetilde{{M}}}, \|\cdot\|_{\infty})}{n}} d\eta \right\}\\
	& \leq \inf_{\gamma \geq 0} \left\{4 \gamma + 12 \int_{\gamma}^{\widetilde{{M}}} 
	\sqrt{\frac{K \widetilde{M}^{d/2}}{n \eta^{d/2}}} d\eta \right\}.
	\end{align*}
	The second inequality is standard, and the third is from substituting the covering number bound from Lemma \ref{lem:vdvw 2.7.1}. The integral is different for $d \neq 4$ and $d = 4$. In the first case,
	\begin{align*}
	\mathcal{R}_n({\mathcal{F}_{\widetilde{{M}}}}) & \leq \inf_{\gamma \geq 0} \left\{4 \gamma + 12 \frac{\sqrt{K} \widetilde{M}^{d/4} \left(4\widetilde{M}^{1 - d/4} - 4 \gamma^{1 - d/4}\right)}{\left(4 - d\right) \sqrt{n}} \right\},
	\end{align*}
	and the infimum is achieved at $\gamma = 81^{1/d} K^{2/d} \widetilde{M} n^{-2/d}$. When $d = 4$, 
	\begin{align*}
	\mathcal{R}_n({\mathcal{F}_{\widetilde{{M}}}}) & \leq \inf_{\gamma \geq 0} \left\{4 \gamma + 12 \frac{\sqrt{K} \widetilde{M}}{\sqrt{n}}( \log{\widetilde{M}} - \log{\gamma} ) \right\},
	\end{align*}
	which is minimized at $\gamma = {3 \sqrt{K} \widetilde{M}} n^{-1/2}$. Substituting in $\gamma$ and $\widetilde{M}$ gives us
	\begin{align*}
	\widetilde{R} :=
		\begin{dcases}
			\frac{4n ^{1/{(2\widetilde{d})}}\left(-\frac{12 \sqrt{K}}{\sqrt{n}} + 81^{1/d} d K^{2/d}  n^{-2/d}\right)}{d - 4} & : d \neq 4\\
			\frac{6 \sqrt{K} n ^{1/{(2\widetilde{d})}}(2 + \log{n} - \log{9} - \log{K} ) }{\sqrt{n}} & : d = 4,
		\end{dcases}
	\end{align*}
	so that $\mathcal{R}_n({\mathcal{F}_{\widetilde{{M}}}}) \leq \widetilde{R}$.\\

	Set
	\begin{align*}
	\varepsilon := 4\widetilde{L}_L \widetilde{R} + 7\widetilde{L}_{max}\sqrt{\frac{\log(8/\delta)}{2n}}.
	\end{align*}
	Each term goes to zero, so $\lim_{n \rightarrow \infty} \varepsilon = 0$. Additionally, ${\partial \varepsilon}/{\partial n} = O\left(-n^{-\left({2\widetilde{d} + 1}\right)/\left({2\widetilde{d}}\right)}\right)$. If $n$ is sufficiently large, ${\partial \varepsilon}/{\partial n} < 0$, and $\varepsilon$ is decreasing in $n$. Finally, applying Lemma \ref{lem:bartlettmendelson} yields the first part of the theorem.\\
	
	To strengthen the result to almost-sure convergence, we appeal to the Borel-Cantelli lemma. It is enough to show that
	\begin{align*}
	\sum_{n}^{} \mathbb{P} \left[\sup_{f \in \mathcal{F}_{\widetilde{{M}}}} \left|R(f) - \widehat{R}(f)\right| > \varepsilon' \, \middle| \, \mathcal{H} \right] + \sum_{n}^{} e^{-n^{2/\widetilde{d}}/{2 \sigma^2}} < \infty,
	\end{align*}
	where $\varepsilon' > 0$ is an arbitrary, fixed value. The second series converges by comparison with the integral
	\begin{align*}
	\int_{0}^{\infty} e^{-n^{2/\widetilde{d}}/{2 \sigma^2}} dn = \left(2 \sigma^2\right)^{\widetilde{d}/2} \Gamma\left(1 + {\widetilde{d}}/{2}\right).
	\end{align*}
	
	Each term in the first series is bounded above by $\min \set{1, \delta}$, with $\delta$ satisfying $\varepsilon' = 4L_L \mathcal{R}_{n}(\mathcal{F}_{\widetilde{{M}}}) + 7L_{max}\sqrt{{\log(8/\delta)}/{2n}}$. For a given $\varepsilon'$, a solution does not exist if $n$ is too small. When $n$ is large enough, we have the following:
	\begin{align*}
	\frac{\delta}{8} & = \exp\left\{-2n\left(\frac{\varepsilon' - 4L_L \mathcal{R}_{n}(\mathcal{F}_{\widetilde{{M}}})}{7L_{max}}\right)^2\right\}\\
	& \leq \exp\left\{-2n\left(\frac{\varepsilon' - 4\widetilde{L}_L \widetilde{R}}{7\widetilde{L}_{max}}\right)^2\right\}\\
	& := \widetilde{\delta}
	\end{align*}
	The second line follows because $\widetilde{L}_L \widetilde{R} \rightarrow 0$ and $\partial{\left(\widetilde{L}_L \widetilde{R}\right)}\big/\partial n < 0$. Eventually, $0 < \widetilde{L}_L \widetilde{R} < \varepsilon'$ and $\left(\varepsilon' - 4\widetilde{L}_L \widetilde{R}\right)^2 \leq \left(\varepsilon' - 4L_L \mathcal{R}_{n}(\mathcal{F}_{\widetilde{{M}}})\right)^2$.\\

	Asymptotically, $\log \widetilde{\delta} = O\left(-n^{1-4/\widetilde{d}}\right)$. Furthermore, its derivative is $O\left(-n^{-4/\widetilde{d}}\right)$, so $\widetilde{\delta}$ is decreasing for large $n$. Since
	\begin{align*}
	\int_{0}^{\infty} e^{-n^{1 - 4.1/\widetilde{d}}} dn = \left(1 - 4.1/\widetilde{d}\,\right)^{-1} \Gamma\left\{\left(1 - 4.1/\widetilde{d}\,\right)^{-1}\right\},
	\end{align*}
	the integral test shows the tail of $\sum_{n} \widetilde{\delta}$ is finite, proving
	\begin{align*}
	\sup_{f \in \mathcal{F}_{\widetilde{{M}}}} \left|R(f) - \widehat{R}(f)\right| \xrightarrow{a. s.} 0.
	\end{align*}
\end{proof}

Finally, the following theorem establishes almost sure convergence of the empirical risk minimizer.
\begin{thm}
	\label{thm:pointwisecvgerm}
	Let $X \sim P_X$, where $P_X$ has density $p$ on $B^d$ such that $0 < c \leq \inf_x p$ for some constant $c$. Let $f^*$ be the true regression function in that observations follow $S_{Y|X} \sim \mathcal{N}(f^*(X), \sigma^2)$. Suppose we set  $\widetilde{M} := n ^{1/{(2\widetilde{d})}}$, where $\widetilde{d} := \max \set{d, 5}$, and let $\widehat{f} \in \mathcal{F}_{\widetilde{M}}$ be the empirical risk minimizer.
	\begin{enumerate}[label=(\roman*)]
		\item For $0 < \delta < 1$,
		\begin{align*}
		\mathbb{P} \left[\sup_{x \in B^d} \left|\widehat{f} - f^*\right| < \beta \right] > (1 - \delta)(1 - e^{-n^{2/\max \set {d, 5}}/{2 \sigma^2}}),
		\end{align*}
		where $\beta$ is a monotonically-decreasing function of $n$ for large enough $n$ and $\lim\limits_{n \rightarrow \infty} \beta = 0$.
		
		\item \begin{align*}
		\sup_{x \in B^d} \left|\widehat{f} - f^*\right| \xrightarrow{a. s.} 0.
		\end{align*}
	\end{enumerate}
\end{thm}

\begin{proof}
	We again condition on the event $\mathcal{H} := \set{\max_{1 \leq i \leq n} \left|\xi_i\right| \leq \sigma \sqrt{2 \log 2n} + n^{1/\widetilde{d}}}$. However, we now set $\widetilde{M} := n^{1/(16\widetilde{d}^2)}$. The conclusions of Theorem \ref{thm:maintheoremERM} still hold with the appropriate modifications made to any constants depending on $\widetilde{M}$. To relate the uniform risk bound to the difference between $\widehat{f}$ and $f^*$, we start by decomposing the risk. With probability at least $(1 - \delta)(1 - e^{-n^{2/\widetilde{d}}/{2 \sigma^2}})$ over the sample,
	\begin{align*}
	\E_X\left[\left(\widehat{f} - f^*\right)^2\right] & = R(\widehat{f}) - R(f^*)\\
	& \leq  \left|R(\widehat{f}) - \widehat{R}(\widehat{f})\right| + \left|\widehat{R}(f^*) - R(f^*)\right|\\
	& \leq 2 \sup_{f \in \mathcal{F}_{\widetilde{{M}}}} \left|R(f) - \widehat{R}(f)\right|\\
	& < 2 \varepsilon.
	\end{align*}
	Combining this with Chebyshev's inequality, we have
	\begin{align*}
	\mathbb{P}_X \left[\left|\widehat{f} - f^*\right| > \alpha \right] & < {\E_X\left[\left(\widehat{f} - f^*\right)^2\right]} \bigg/ \alpha^{2}\\
	& < 2 \varepsilon \alpha^{-2}
	\end{align*}
	for $\alpha > 0$. In other words, $\widehat{f}$ lies within a tube of radius $\alpha$, except on a set $A \subset B^d$ such that $P_X(A) < 2 \varepsilon \alpha^{-2}$.\\ 
	
	Let $h := \sup_{x \in A} \left|\widehat{f} - f^*\right| - \alpha$. Because $\widehat{f}$ and $f^*$ are Lipschitz, $h$ is constrained by the inequality
	\begin{align*}
	\frac{\left(f^*(x) + M^*r + \alpha + h\right) - \left(f^*(x) + \alpha\right)}{r} \leq \widetilde{M},
	\end{align*}
	where $x$ is on the boundary of $A$, and $r$ is the inradius of $A$. This implies that $h \leq \widetilde{M}r$. We can maximize this by taking $A$ to be the $d-$dimensional ball of radius
	\begin{align*}
	\widetilde{r} := \left(\frac{2 \varepsilon \alpha^{-2} \Gamma\left(1 + \frac{d}{2}\right)}{c \pi^{d/2}}\right)^{1/d},
	\end{align*}
	where $c$ is a constant bounding the density $p$ away from zero. This shows
	\begin{align*}
	\sup_{x \in B^d} \left|\widehat{f} - f^*\right| < \widetilde{M}\widetilde{r} + \alpha.
	\end{align*}
	
	Set $\alpha := n^{-1/\left(10\widetilde{d}\right)}$. Then $\widetilde{M}\widetilde{r} = O(n^\rho)$ and ${\partial \left(\widetilde{M}\widetilde{r}\right)}/{\partial n} = O(-n^{\rho - 1})$, where
	\begin{align*}
	\rho :=
		\begin{cases}
			-{3}\big/\left({50d}\right) + {1}\big/{400} & : d \leq 5\\
			\left({5- 59d}\right)\big/\left({80d^3}\right) & : d > 5.
		\end{cases}
	\end{align*}
	Now, defining $\beta := \widetilde{M}\widetilde{r} + \alpha$ gives the first part of the theorem.\\	
	
	Almost-sure convergence follows from a similar argument as part (ii) of Theorem \ref{thm:maintheoremERM}. It suffices to show that, for arbitrary $\beta' > 0$,
	\begin{align*}
	\sum_{n}^{} \mathbb{P} \left[\sup_{x \in B^d} \left|\widehat{f} - f^*\right| > \beta' \, \middle| \, \mathcal{H} \right] + \sum_{n}^{} e^{-n^{2/\widetilde{d}}/{2 \sigma^2}} < \infty,
	\end{align*}
	where we have already shown convergence of the second series. Let
	\begin{align*}
	\varepsilon' := \left(\frac{\beta' - \alpha}{\widetilde{M}}\right)^d\left(\frac{c \pi^{d/2}}{\Gamma\left(1 + d/2\right)}\right)\left(\frac{\alpha^2}{2}\right)
	\end{align*}
	be the solution when setting $\beta' = \widetilde{M} \widetilde{r} + \alpha$ and solving for $\varepsilon$. For fixed $\beta'$ and large $n$, there is a corresponding $\varepsilon' > 0$. Note that $\varepsilon'$ is not fixed, but decreasing in $n$. In fact, $\varepsilon' = O\left(n^{-\left(5d + 16 \widetilde{d}\,\right)/\left(80 \widetilde{d}\,^2\right)}\right)$.  Since $\widetilde{L}_L \widetilde{R} = O\left(n^{-119/400}\right)$ for $d < 5$ and $O\left(n^{-\left(16d - 1\right)/\left(16 d^2\right)}\right)$ for $d \geq 5$, eventually $0 < \widetilde{L}_L \widetilde{R} < \varepsilon'$. Therefore,
	\begin{align*}
	\widetilde{\delta} := 8 \exp\left\{-2n\left(\frac{\varepsilon' - 4\widetilde{L}_L \widetilde{R}}{7\widetilde{L}_{max}}\right)^2\right\}
	\end{align*}
	is an upper bound for the tail of the first series. Observe that $\log \left(\widetilde{\delta}/8\right) = O\left(-n^{1 - d/\left(8 \widetilde{d}^2\right) - 22/\left(5 \widetilde{d}\right)}\right)$ and $\partial\left(\log \left(\widetilde{\delta}/8\right)\right)/\partial n = O\left(-n^{- d/\left(8 \widetilde{d}^2\right) - 22/\left(5 \widetilde{d}\right)}\right)$. Comparison with the integral
	\begin{align*}
	\int_{0}^{\infty} e^{-n^{1 - d/\left(8 \widetilde{d}^2\right) - 22.1/\left(5 \widetilde{d}\right)}} dn & = 1 \bigg/ \left(1 - d/\left(8 \widetilde{d}\,^2\right) - 22.1/\left(5 \widetilde{d}\,\right)\,\right)\\ 
	& \times \Gamma\left\{\left(1 - d/\left(8 \widetilde{d}\,^2\right) - 22.1/\left(5 \widetilde{d}\,\right)\,\right)^{-1}\right\}
	\end{align*}
	is enough to give almost-sure uniform convergence of the empirical risk minimizer.
\end{proof}

\subsection{Vaidya's Algorithm} \label{sec:vaidya}
Given the value of $M$, it remains to solve the optimization problem \eqref{eq:alg}. This is a convex program over a set that is not a polytope, and can be solved using interior point methods with a barrier function constructed for the $\dot{C}^{1,1}(E)$ seminorm constraint.  However, such methods would have no guarantees on the convergence time without deriving properties of the barrier such as self-concordance \citep{nesterov1994interior}, a topic we leave open for future study.  Instead, we detail a solution by way of Vaidya's algorithm \citep{vaidya1996new}, making use of a slight modification of an efficient implementation provided by \cite{anstreicher1997vaidya}.  Vaidya's algorithm is a cutting-plane method which seeks a feasible point in an arbitrary convex set $K \subset S_0 := \set{x \in \R^k \mid \|x\|_\infty \leq \rho}$ (note that \citet*{anstreicher1997vaidya} assumes $\rho = 1$).  The set $K$ is specified by a separation oracle: given a point $y \in \R^k$, the oracle either certifies that $y \in K$, or returns a separating hyperplane between $y$ and $K$ (i.e., a vector $w$ such that $K \subset \set{x \mid w \cdot (x - y) \leq 0}$). The algorithm initializes with an interior point $x_0 = 0$ and polytope $S_0$, and maintains a polytope $S_t \supset K$ and an interior point $x_t$ of $S_t$ at each iteration $t$, where $S_t$ is defined via the separation oracle.  At each iteration $t$, a constraint is either added or deleted, and the polytope $S_t$ is specified by no more than $201k$ constraints throughout the algorithm.  One of the strengths of Vaidya's algorithm is that it comes with complexity guarantees: that after $T = O(k(L + \log \rho))$ calls to the separation oracle, we have
$$
\mathrm{Vol}_k(S_T) < \mathrm{Vol}_k(2^{-L} B^k),
$$
where $L = \Omega(\log k)$ is a user-specified constant.  Thus, the algorithm certifies that if no feasible point is found within $T$ iterations, the volume of $K$ is less than that of a $k$-dimensional ball of radius $2^{-L}$.  We remark that the value of $T$ in our case (of $\rho \neq 1$ in general) is easily determined via an argument along the lines of Lemma 3.1 of \cite{anstreicher1997vaidya}, and is given as 

\begin{equation}\label{eq:num_it}
T \geq \frac{k\left[1.4 L + 2 \log k + 2 \log (1 + 1/\varepsilon) + 0.5 \log \left(\frac{1+\tau}{1-\varepsilon}\right) + 2\log (\rho) - \log(2)\right]}{\Delta V},
\end{equation}
where $\varepsilon = 0.005$ and $\tau = .007$ are parameters of the algorithm, and $\Delta V = 0.00037$. The algorithm uses a total of $O(k(L + \log \rho)\xi + k^4 (L + \log \rho))$ operations using standard linear algebra, where $\xi$ is the cost of evaluating the separation oracle.

The feasibility algorithm may be applied to minimize an arbitrary convex function $g(\cdot)$ as follows.  The minimization problem is essentially a feasibility problem in which we seek a point $\widehat{x}$ in the set $K \cap \set{x \mid g(x) - g(x^\star)\leq \gamma}$, where $\gamma > 0$ is an error tolerance and $x^\star$ is any minimizer of $g$ on $K$.  If we find a point $y \in K$, we instead use the oracle specified by any subgradient $w \in \partial g(y)$ to localize an optimal solution.  If $0 \in \partial g(y)$, then $y$ is an optimal point, and we are done.  Otherwise, we use the hyperplane  $\set{x \mid w \cdot (x - y) \leq 0}$ within which the set $\set{x \mid g(x) \leq g(y)}$ is contained, and proceed as in the feasibility case.  If an optimal $x^\star$ was not found in $T$ iterations, we find an approximate solution as follows.  Let $\mathcal{T} \subset \set{1, 2, \ldots, T}$ denote the steps for which an $x_t \in K$ was found, after $T$ iterations we return 
\begin{equation}\label{eq:approx_soln}
\widehat{x}_T \in \underset{x_s, \,s \in \mathcal{T}}{\mathrm{argmin}} \;g(x_s).
\end{equation}
Note that $g(x^\star)$ is not known, so we cannot directly evaluate whether any estimate $\widehat{x}_T \in K$ satisfies the error tolerance.  However, given information on the geometry of $K$ and on the objective function, we may choose $T$ to guarantee that this is the case. 
Fix $x^\star$ to be any optimal solution, and define $K_\varepsilon(x^\star) := x^\star + \varepsilon (K - x^\star)$ which contains the points in $K$ in a small neighborhood around $x^\star$.   Now let $x^\star_\varepsilon$ denote the worst possible $x \in K_\varepsilon(x^\star)$ in terms of having the largest value of $g$ over all possible optimal solutions $x^\star$.  \cite{nemirovski1995information} defines an $\varepsilon$-solution to be any $x \in K$ such that 
$$
g(x) \leq g(x^\star_\varepsilon),
$$
and provides the following theorem.
\begin{thm}\label{thm:nemirovski}
Assume that after $T$ steps the method has not terminated with an optimal solution.  Then given that $\mathcal{T} \neq \emptyset$, any solution $\widehat{x}_T$ of equation \eqref{eq:approx_soln} is an $\varepsilon$-solution for any $\varepsilon$ such that
$$
\varepsilon^k > \frac{\mathrm{Vol}_k(S_T)}{\mathrm{Vol}_k(K)}.
$$
If the function $g$ is convex and continuous on $K$, then any $\varepsilon$-solution $x$ satisfies
$$
g(x) - g(x^\star) \leq \varepsilon\left(\sup_{x \in K} g(x) - g(x^\star)\right).
$$
\end{thm}\noindent

Before finding the requisite number of iterations $T$, let us first derive the separation oracles, starting with the oracle for $K_1(M) := \set{P \mid \|P\|_{\dot{C}^{1, 1}(E)} \leq M}$. Presume at the current step of the algorithm, we have a set of function values and gradients $\set{f(a), D_a f}_{a \in E}$. which generate a candidate $1$-field $P$.  By Theorem \ref{thm:legruyer} and equation \eqref{eq:pfunc}, we may find the $a, b \in E, a \neq b$ such that $\Gamma^1(P; E) = \|P\|_{\dot{C}^{1,1}(E)}$ in $n(n-1)/2$ operations.  Thus, to determine whether the $1$-field at the current step is contained in the constraint set, we simply check if $\|P\|_{\dot{C}^{1, 1}(E)} \leq M$.  Otherwise, we must return a separating hyperplane in the space of $1$-fields.  Let $a^\star, b^\star \in E, \,a^\star \neq b^\star$ be any elements of $E$ that solve \eqref{eq:legruyerfunc2}, with $a^\star$ denoting the first element of $E$ in the numerator of \eqref{eq:legruyerfunc2}. Specifying the separating hyperplane requires finding the $x \in \R^d$ that solves \eqref{eq:legruyerfunc2}, that is, $x \in \R^d$ such that
\begin{equation}\label{eq:pfunc_noab}
\Gamma^1(P; E) = 2  \sup_{x \in \bar{B}^d\left(\frac{a^\star+b^\star}{2}, \frac{|a^\star-b^\star|}{2}\right)} \frac{P_{a^\star}(x) - P_{b^\star}(x)}{|a^\star - x|^2 + |b^\star - x|^2}.
\end{equation}
Equation \eqref{eq:pfunc_noab} is a nonlinear fractional program, and is equivalent to minimizing the ratio
\begin{equation}\label{eq:fracprog}
R(x):= \frac{N(x)}{D(x)},
\end{equation}
where $N(x) = |a^\star - x|^2 + |b^\star-x|^2$ and $D(x) = 2(P_{a^\star}(x) - P_{b^\star}(x)) > 0$. Here, we additionally know that the minimizer of \eqref{eq:fracprog} attains the optimal value $1/\Gamma^1(P; E)$ due to equation \eqref{eq:pfunc}. \citet*{jagannathan1966some} and \citet*{dinkelbach1967nonlinear} showed that for $N(x)$ continuous, $D(x) > 0$ continuous, the solution to $\min_{x \in \mX} R(x)$ over a compact subset $\mX \subset \R^d$ is $z \in \mX$ if and only if $z \in \mX$ is also an optimal solution for 
$$
\min_{x \in \mX} N(x) - R(z) D(x).
$$
Plugging in the optimal value $R(y) = 1/\Gamma^1(P; E)$ yields the minimization
$$
\min_x |a^\star - x|^2 + |b^\star - x|^2 - \left(\frac{2(P_{a^\star}(x) - P_{b^\star}(x)) }{\Gamma^1(P; E)}\right) .
$$
Thus finding the $x$ which solves \eqref{eq:pfunc_noab} amounts to minimizing a convex quadratic in $x$.  The solution is
\begin{equation}\label{eq:opt_x}
z = \left(\frac{a^\star + b^\star}{2}\right) + \left(\frac{D_{a^\star}f - D_{b^\star}f}{2 \Gamma^1(P; E)}\right).
\end{equation}
The separation oracle for feasibility is thus specified as follows. For a candidate $1$-field $P$, if  $\|P\|_{\dot{C}^{1, 1}(E)}\leq M$, then certify that $P$ is a feasible $1$-field.  Otherwise, separate all other $1$-fields $\widetilde{P} = \set{\widetilde{f}(a), \widetilde{D_a f}}_{a \in E}$ from $P$ via
$$
\set{\widetilde{P} \; \bigg| \; \frac{2(\widetilde{P}_{a^\star}(z) - \widetilde{P}_{b^\star}(z))}{|a^\star - z|^2 + |b^\star - z|^2} \leq \Gamma^1(P; E)},
$$
or equivalently
$$
(\widetilde{f}(a^\star) + \widetilde{D_{a^\star}f}\cdot(z - a^\star)) - (\widetilde{f}(b^\star) + \widetilde{D_{b^\star}f}\cdot(z-b^\star)) \leq \Gamma^1(P; E)\left(|a^\star-z|^2 + |b^\star-z|^2\right).
$$
The separation oracle for $K_1(M)$ is thus equivalent to using the vector 
$$
w_{a^\star, b^\star} = \begin{pmatrix}
1 \\
z-a^\star \\
-1 \\ 
-(z-b^\star)
\end{pmatrix} \in \R^{2(d+1)}
$$
and the scalar $u_{a^\star,b^\star} = \Gamma^1(P; E)\left(|a^\star-z|^2 + |b^\star-z|^2\right)$ to define the hyperplane
\begin{equation}\label{eq:feas_oracle}
\set{v \in \R^{2(d+1)} \mid w_{a^\star, b^\star} \cdot v \leq u_{a^\star, b^\star}}.
\end{equation}
Appropriately padding $w_{a^\star, b^\star}$ and $v$ with zeros over the remaining possible choices of $a, b$ thus defines a $k$-dimensional separating hyperplane, and this separating hyperplane is constructed in $O(n^2 + d)$ operations. 

Note that the objective function in \eqref{eq:alg} is equivalent to
$$
g(P) = \frac{1}{n}|y - f|^2.
$$
To construct the separation oracle for $K_2(\gamma) := \set{P \mid g(P) - g(P^\star) \leq \gamma}$, taking the gradient with respect to $f$, we have $w = \frac{2}{n}\left(f - y\right)$.  Thus, given a current feasible field $P$ with function values $f$, the separating hyperplane is specified as
\begin{equation}\label{eq:opt_oracle}
\set{\widetilde{f} \in \R^n \mid w \cdot \widetilde{f} \leq u},
\end{equation}
where $u = w \cdot f$.  Suitably concatenating $w$ and $\widetilde{f}$ with zeros to form vectors in $\R^k$ thus specifies the separation oracle for $K_2(\gamma)$, and requires $O(n)$ operations to evaluate.    

Now, to find the requisite number of iterations $T$ to find a feasible point in $K_1(M) \cap K_2(\gamma)$, we sandwich the set $K_1(M)$ with Euclidean balls.  The next result characterizes the Euclidean ball inside $K_1(M)$.
\begin{lem}\label{lem:inner_ball}
Assume $|a - b| \geq r > 0$ for all $a, b \in E, \, a \neq b$. Then 
$$
\rho_1 B^k \subset K_1(M),
$$
where 
\begin{equation}\label{eq:inner_const}
\rho_1 = \left(\frac{r^2M}{8(1+r)}\right)\sqrt{n}.
\end{equation}
\end{lem}
\begin{proof}
Let $P = \set{f(a), D_af}_{a \in E}$ be any $1$-field, and assume that 
$$
|P_a| = \sqrt{|f(a)|^2 + |D_a f|^2} \leq \rho
$$ 
for all $a \in E$, so $P$ is represented by a vector in $(\rho\sqrt{n}) B^k$.  Note that the numerator of \eqref{eq:legruyerfunc2} may be written as 
\begin{equation}\label{eq:pdiff}
P_a(x) - P_b(x) = (f(a) - f(b)) + \frac{1}{2}\left(D_af + D_bf\right)\cdot(b-a) + (D_af - D_bf)\cdot \left(x -\frac{a+b}{2}\right).
\end{equation}
Thus,
$$
\begin{aligned}
|P_a(x) - P_b(x)| &\leq 2\rho + \rho|b-a| + 2\rho\bigg|x - \frac{a+b}{2}\bigg|\\
&\leq 2\rho(1 + |b-a|),
\end{aligned}
$$
where we have used the fact that $x \in \bar{B}^d\left(\frac{a+b}{2}, \frac{|a-b|}{2}\right)$ in \eqref{eq:legruyerfunc2}.  The denominator is minimized at $x = \frac{a+b}{2}$ with minimal value $|a-b|^2/2$.  Thus
$$
\begin{aligned}
\Gamma^1(P; E) &\leq \frac{8\rho(1+|b-a|)}{|b-a|^2} \\
&\leq 8\rho(r^{-1} + r^{-2}).
\end{aligned}
$$
It follows that $\|P\|_{\dot{C}^{1, 1}(E)} \leq M$ if $\rho \leq \frac{Mr^2}{8(1+r)}$.
\end{proof} \noindent
We now derive a bounding ball for $K_1(M)$ which indicates the value of $\rho$ to use in Vaidya's algorithm.  

\begin{lem} \label{lem:outer_ball}
Assume $E$ is an $\varepsilon-$net of $B^d$, the unit ball in $\R^d$, where $\varepsilon < 1/10$. Suppose for all distinct $a, b\in E$, $|a - b| > r$. Then 
$$
K_1(M) \subset \rho_2 B^k,
$$
where 
\begin{equation}\label{eq:outer_const}
\rho_2 = \sqrt{n}\left[ \frac{|y|^2}{n} + 4\left(\frac{10|y|}{r} + \frac{5M}{2}\right)^2\right]^{1/2}.
\end{equation}
\end{lem}
\begin{proof}
Let $\prod_f$ denote the projection of any subspace of $\R^k$ onto the $n$-dimensional subspace corresponding to $D_a f = 0$ for all $a \in E$.  If $y \in \prod_f K_1(M)$, then an optimal solution is given by $\set{y(a), 0}_{a \in E}$.  Thus we may presume that $y \notin \prod_f K_1(M)$, whence 
$$
|f| \leq |y|.
$$
To bound $|D_a f|$ given $a \in E$, note that by equation \eqref{eq:legruyerfunc2} we have 
$$
\frac{2|P_a(x) - P_b(x)|}{|a-x|^2 + |b-x|^2} \leq M
$$
for any $b \in E\setminus{\set{a}}$ and $x \in B^d\left(\frac{a+b}{2}, \frac{|a-b|}{2}\right)$.  Choosing $x = b$ implies that
$$
\begin{aligned}
|D_a f \cdot (b-a)| &\leq |f(a) - f(b)| + \frac{M}{2}|a-b|^2 \\
&\leq 2|y| + \frac{M}{2}|a-b|^2.
\end{aligned}
$$
By the conditions of the lemma, given $a$, there exist $b \in E\setminus\set{a}$ such that $|D_a f \cdot (b-a)|\geq \frac{|D_a f||(b_1-a)|}{10}$.  It follows that
$$
|D_a f| \leq \frac{20|y|}{r} + {5M}. 
$$
Thus 
$$
\begin{aligned}
|P|^2 &= |f|^2 + \sum_{a \in E} |D_a f|^2 \\
&\leq |y|^2 + 4n\left(\frac{10|y|}{r} + \frac{5M}{2}\right)^2.
\end{aligned}
$$
\end{proof}\noindent
We arrive at the number of iterations required such that $g(P) - g(P^\star) \leq \gamma$.
\begin{thm}
Let $\gamma > 0$ be an error tolerance parameter, let $P^\star$ be any optimal solution to \eqref{eq:alg}, and assume $0 < r \leq |a - b| \leq R$ for all $a, b \in E$.  Applying Vaidya's algorithm for minimization as in \eqref{eq:approx_soln} using the separation oracles specified in \eqref{eq:feas_oracle} and \eqref{eq:opt_oracle} yields an approximate solution $\widehat{P}_T$ to \eqref{eq:alg} such that
$$
g(\widehat{P}_T) - g(P^\star) \leq \gamma
$$
where we choose
$$
L \geq \log_2\left(\frac{4|y|^2}{n\gamma \rho_1} \right),
$$
with $\rho_1$ as stated in lemma \ref{lem:inner_ball} and $T$ is given in equation \eqref{eq:num_it} using $\rho = \rho_2$ from lemma \ref{lem:outer_ball}.
\end{thm}
\begin{proof}
Recall that if $y \in \prod_f K_1(M)$, then an optimal $1$-field is returned without calling Vaidya's algorithm via $\set{y(a), 0}_{a \in E}$.  Thus we may presume that $|f| \leq |y|$ on $\prod_f K_1(M)$.  Thus
$$
g(P) = \frac{1}{n}|y - f|^2 \leq \frac{4|y|^2}{n}
$$
for any $P \in K_1(M)$.  Thus we set $\varepsilon = \frac{n\gamma}{4|y|^2}$, and apply theorem \ref{thm:nemirovski}.  From lemma \ref{lem:inner_ball}, we have that $\mathrm{Vol}_k(K_1(M)) \geq \rho_1^k \mathrm{Vol}_k(B^k)$.  Since $\mathrm{Vol}_k(S_T) < 2 ^{-kL} \mathrm{Vol}_k(B^k)$, it suffices to choose $L$ such that
$$
2^{-kL}\rho_1^{-k} \leq \left(\frac{n \gamma}{4|y|^2}\right)^k,
$$
from which the statement results.  
\end{proof}

\subsection{Constraining the \boldmath{$\dot{C}^{1,1}$}-Seminorm Rather Than the \boldmath{$C^{1, 1}$}-Norm}
In the previous section, we showed how to use Vaidya's algorithm to solve the optimization problem \eqref{eq:alg} that is central to this paper. Before we use the output 1-field to actually construct the interpolant, we need to show that the risk bounds derived in Section \ref{sec:erm} apply to our optimization scheme. The only potential conflict is that our solution to \eqref{eq:alg} involves constraining the $\dot{C}^{1, 1}$-seminorm of functions defined on a discrete set $E \subset \R^d$; however, the risk bounds given in Section \ref{sec:erm} are based on the overall $C^{1, 1}$-norm of functions defined on the unit ball $B^d \subset \R^d$ (the overall norm is the maximum of the $\dot{C}^{1, 1}$-seminorm, the Euclidean norm of the gradient, and the absolute value of the function values). In this section, we show that as long as the sample size is large enough, the $\dot{C}^{1, 1}$-seminorm $\|\cdot\|_{\dot{C}^{1, 1}(E)}$ determines the $C^{1, 1}$-norm $\|\cdot\|_{C^{1, 1} (B^d)}$  in our setup with high probability. 

Recall that for a finite set of points $E \in \R^d$ and a function $f: E \rightarrow \R$, norms and seminorms of $f$ are defined in terms of their analogues for continuous-domain extensions of $f$. Specifically,
\begin{align*}
\|f\|_{\dot{C}^{1, 1}(E)} &:= \inf\set{\Lip(\nabla \widetilde{f}) \mid \widetilde{f}(a) = f(a) \text{ for all } a \in E}, \text{ where } \widetilde{f} \in \dot{C}^{1, 1}\left(B^d\right),\\
\|f\|_{C^{0}(E)} &:= \inf\set{\sup_{x \in B^d}|\widetilde{f}(x)| \mid \widetilde{f}(a) = f(a) \text{ for all } a \in E}, \text{ where } \widetilde{f} \in C^{0}\left(B^d\right),\\
\|f\|_{C^{1}(E)} &:= \inf\set{\sup_{x \in B^d}\|\nabla \widetilde{f}(x)\| \mid \widetilde{f}(a) = f(a) \text{ for all } a \in E}, \text{ where } \widetilde{f} \in C^{1}\left(B^d\right), \text{ and}\\
\|f\|_{{C}^{1, 1}(E)} &:= \max \set{\|f\|_{C^{0}(E)},\, \|f\|_{C^{1}(E)},\, \|f\|_{\dot{C}^{1, 1}(E)}}.
\end{align*}
Clearly, $\|f\|_{C^{0}(E)}$ only depends on the specified values of $f$ on $E$. Kirszbraun's Theorem states that the $C^1$-norm is also completely determined by these values; $\|f\|_{C^{1}(E)}$ is equal to the maximum slope between pairs of points in $E$.

The main results we wish to establish are Theorem \ref{c11dotmain} (including two intermediate lemmas) and Theorem \ref{c11dotmain2}. Let $f^{*}$ be the true function that appears in the generative process. Let $\|f^{*}\|_{C^{1, 1}(B^d)}\leq M^{*}$. Let a set $X$ containing $n$ random points be chosen i.i.d from $\mathcal P$, which we assume has a density $\rho(x)$ with respect to the Lebesgue measure on $B^d$ and a minimum density $\rho_{min}$.
 Let $y_i =  f^{*}(x_i) + \xi_i$, where $x_i \in X_0$ and $\xi_i$ is a Gaussian with mean $0$ and variance $\sigma^2$ that is independent of all the other $\xi_j$. Set $\widetilde{M} := n ^{1/{(2\widetilde{d})}}$, where $\widetilde{d} := \max \set{d, 5}$. We will denote in this section by $C_d$ constants depending only on $d$. Suppose we project $y$ onto the set of all functions $f$ such that $\|f\|_{\dot{C}^{1, 1}(X)} \leq \widetilde{M}$. In Theorem \ref{c11dotmain}, we prove for a large-enough sample that the $C^0(X)$- and $C^1(X)$-norms of the projection are less than $\widetilde{M}/2$. Furthermore, in Theorem \ref{c11dotmain2} we show that the extension of this projection to the unit ball has $C^{1, 1}(B^d)$-norm no more than $\widetilde{M}$. This is enough to show that the sample complexity results are compatible with the construction of the interpolant in Sections \ref{sec:vaidya} and \ref{sec:wells}.

\begin{thm}\label{c11dotmain}
 Let $K \subseteq L^2(X)$ be the closed convex set of all functions $f$ such that
  \begin{align*}
\|f\|_{\dot{C}^{1, 1}(X)} \leq \widetilde{M}.
  \end{align*}
   Let $h$ be the projection of $y$ onto $K$ with respect to the Hilbert space $L^2(X)$. Then when $n$ is sufficiently large,  with  probability at least $1 - \exp(-n^{1/100})$, 
\begin{align*}\max\left(\|h\|_{C^0(X)}, \|h\|_{C^1(X)}\right) < \widetilde{M}/2.\end{align*}
\end{thm}

\begin{proof}

Let the unit cube $\Box$ be covered (up to a set of measure $0$) by open cubes $\Box_i$ centered on the lattice $\delta_1 \mathbb Z^d$ and having side length $\delta_1$. We will prove the result by comparing the functions in $K$ to piecewise affine functions defined on the $\Box_i$. The error between $f \in K$ and the piecewise affine approximation is dependent on $\|f\|_{C^{1, 1}}$ and is arbitrarily small for large enough $n$ by an appropriate choice of $\delta_1$. First, let us focus on one subcube $\Box_0$ contained entirely in $B^d$. Let $X_0 = \Box_0 \cap X$ consist of $n_0$ points, where \begin{align*}n_0 = \text{Bin}\left(n, \int_{\Box_0} \rho(x)dx\right);\end{align*} i.e., $n_0$ is a binomial random variable corresponding the the number of heads in $n$ tosses of a coin whose probability of coming up heads is $ \int_{\Box_0} \rho(x)dx$. It is easy to show that $n_0$ is bounded below with high probability. Note that 
\begin{align*}\E[n_0] = n \int_{\Box_0} \rho(x)dx \geq n\rho_{min}\delta_1^d.\end{align*} This implies that for $\delta_2 \in (0, 1)$,  \begin{equation} \label{eq:chernoff} \mathbb{P}[n_0 > (1 - \delta_2)\E[n_0]] \geq 1 - \exp(- \delta_2^2\E[n_0]/2)\end{equation} by the Chernoff bound. 

Let $\text{aff}(X_0)$ denote the space of functions $y:X_0 \rightarrow \mathbb{R}$ that have the form $y = v \cdot x + c$, where $v \in \mathbb{R}^d$  and $c \in \mathbb{R}$ are independent of $x$.  Similarly, let 
$\text{aff}(\Box_0)$ denote the space of functions $y:\Box_0 \rightarrow \mathbb R$ that have the form $y = v \cdot x + c$, where $v \in \mathbb R^d$  and $c \in \mathbb R$ are independent of $x$. Given a function $g:X_0 \rightarrow \mathbb R$, let \begin{align*}\|g\|_2 := \|g\|_{L^2(X_0)}^2 = \sum_{x_i \in X_0} |g(x_i)|^2.\end{align*} Now, let $\mu_0$ denote the measure ${\mathcal P}|_{\Box_0}n$. Given a function $g:\Box_0 \rightarrow \mathbb R$, let \begin{align*} \|g\|_{L^2(\mu_0)}^2 = \int_{x \in \Box_0} g(x)^2\mu_0(dx).\end{align*} 
The following lemma relates these two norms for affine functions.

\begin{lem}\label{lem:3.2}
With probability at least $ 1 - C d^2 \exp ( - c d^{-4}n_0)$,                    
\begin{align*} \sup_{0 \neq f \in \text{aff}(\Box_0)} \frac{\left| \|f\|_{L^2(X_0)} - \|f\|_{L^2(\mu_0)}\right|}{\|f\|_{L^2(\mu_0)}} < \frac{1}{2}. \end{align*}
\end{lem}
\begin{proof}
Let $\{a_1, \dots, a_d\}$ be an $L^2(\mu_0)-$orthonormal basis. Then, by the Chernoff bound,  we have \begin{align*}\mathbb{P}\left[\left|\langle a_i, a_j \rangle_{L^2(X_0)} - \delta_{ij}\right| > \frac{1}{C d^2}\right] <C\exp ( - c d^{-4}n_0).\end{align*} 
The Lemma follows from the union bound.

\end{proof}

Let $X_j := \Box_j \cap X$. Let $f_{lin, j}$ be a function in $\text{aff}(X_j)$ for which \begin{align*}\sum_{x_i \in X_j} |f_{lin, j}(x_i) - y_i|^2 = \inf_{ f \in \text{aff}(X_j)} \sum_{x_i \in X_j} |f(x_i) - y_i|^2.  \end{align*} $f_{lin, j}$ is also the projection of $y|_{X_j}$ onto $\text{aff}(X_j)$, denoted $\text{Proj}_{L^2(X_j)} (y, \text{aff}(X_j))$, where the projection is with respect to the Hilbert space $L^2(X_j)$. Let $\text{aff}_{\delta_1}(X)$ denote the space of functions from $X$ to $\mathbb{R}$ that are affine restricted to each piece $\Box_j$. Let $f_{lin}$ be a function consisting piecewise of all the $f_{lin, j}$, i.e., a function in $\text{aff}_{\delta_1}(X)$ for which \begin{align*}\sum_{x_i \in X} |f_{lin}(x_i) - y_i|^2 = \inf_{ f \in \text{aff}_{\delta_1}(X)} \sum_{x_i \in X_0} |f(x_i) - y_i|^2.  \end{align*} In the next lemma, we show that $f_{lin}$ is very close to $f^*$ for large $n$. 

\begin{lem}\label{lemma15} Choose  $ n^{- 1/(1 + d)} < \delta_1 < n^{- 1/(100d)}$ as a function of $n$. For large $n$ 
the following is true with probability at least $1 - \exp(-n^{1/(100d)})$:
\begin{enumerate}
\item \begin{align*} \max\left(\max_i \|f_{lin} - f^*|_{\Box_i}\|^2_{C^0},\; \max_i \delta_1^2 \|f_{lin} - f^*|_{\Box_i}\|^2_{C^1}\right) \leq C_d\frac{\rho_{max}}{\rho_{min}}{M^*}^2\delta_1^4.\end{align*}

\item \label{eq:f2} \begin{equation}\label{eq:f2-1} n^{-1} \|f_{lin} - f^*\|^2_{L^2(X)}  <  C_d{{M^*}^2\delta_1^4}.\end{equation}

\item  \begin{equation}\label{eq:f2-3} n^{-1} \|f_{lin} - f^*\|^2_{L^2(\mu)}  <  C_d{{M^*}^2\delta_1^4}.\end{equation}

\end{enumerate}
\end{lem}
\begin{proof}
Without loss of generality, let the origin be shifted to the center of $\Box_0$. Then, given a function $f \in \text{aff}(\Box_0)$, it can be uniquely expressed as $f = c + f_1 + \dots + f_d $, where $c$ is a constant, $f_i(x) = v_i \cdot x$ and $v_i$ is a scalar multiple of $e_i$, the $i^{th}$ canonical basis vector. 

By Taylor's Theorem, $f^*$ has an affine approximation $g_{lin, 0}$ such that, for all $x_i \in X_0$, $|g_{lin, 0}(x_i) - f^*(x_i)| < \varepsilon_1$, where $\varepsilon_1 := C_d {M^*} \delta_1^2$. If $n$ is large, $\delta_1$ is small, so we can assume that $\varepsilon_1 < 1$. By the triangle inequality,
\begin{equation}\left\|f_{lin, 0} - f^*\right\|_2^{1/2} \leq \left\|\text{Proj}_{L^2(X_0)} (f^*, \text{aff}(X_0)) - f^*\right\|_2^{1/2} +  \left\|\text{Proj}_{L^2(X_0)} (\xi, \text{aff}(X_0))\right\|_2^{1/2}.\end{equation}
The first term on the right-hand side is clearly bounded as
\begin{align*}
\left\|\text{Proj}_{L^2(X_0)} (f^*, \text{aff}(X_0)) - f^*\right\|_2^{1/2}  &\leq \|g_{lin, 0} - f^*\|_2^{1/2}\\
& \leq   \varepsilon_1 \sqrt{n_0}.
\end{align*} 
By Gaussian concentration,
\begin{equation} \mathbb{P} \left[\left\|\text{Proj}_{L^2(X_0)} (\xi, \text{aff}(X_0))\right\|_2^{1/2}  \leq    2 \sqrt{d}n_0^{1/4}\sigma \right] \geq 1 - C\exp(- c d\sqrt{n_0}), \end{equation} 
implying that
\begin{equation}\label{eq:h1} 
\|f_{lin, 0} - f^*\|_{L^2(X_0)} \leq  \left(\varepsilon_1 \sqrt{n_0} + 2 \sqrt{d}n_0^{1/4}\sigma\right)^2
\end{equation} 
with high probability. By Lemma \ref{lem:3.2}, with probability at least 
$$
1 - C\exp(- c d (n\delta_1^d)^{1/2}) - Cd^2 \exp ( - c d^{-4}(n\delta_1^d)),
$$ 
on any $ \Box_i$, 
\begin{align*}\|f_{lin} - f^*\|^2_{L^2(\mu_i)} < 3 \left(C_d {M^*}\delta_1^2\sqrt{n\delta_1^d} + 2 \sqrt{d}(n\delta_1^d)^{1/4}\sigma\right)^2.
\end{align*}

 Let $\mu_{u}$ denote the uniform measure on $\Box$ having Radon-Nikodym derivative $n\rho_{min}$ with respect to the Lebesgue measure. Let $\mu_{u, 0}$ be the restriction of $\mu_u$ to $\Box_0$. Thus, for $f \in \text{aff}(\Box_0)$,

\begin{eqnarray}
\|f\|^2_{L^2(\mu_0)} & \geq &  \|f\|^2_{L^2(\mu_{u,0})}\\
&  = & n_0c^2 + \sum_i \|f_i\|^2_{L^2(\mu_{u,0})}\\
& \geq & \frac{n_0\rho_{min}}{12 \rho_{max}}   \max\left(\|f\|^2_{C^0},\, \delta_1^2 \|f\|^2_{C^1}\right).
\end{eqnarray}


Therefore,  with probability at least $ 1 - \delta_1^{-d}\left(C\exp(- c d (n\delta_1^d)^{1/2}) - Cd^2 \exp ( - c d^{-4}(n\delta_1^d))\right)$, 
\begin{align*}
&\delta_1^{-d}\left(3 {M^*}\delta_1^2\sqrt{n\delta_1^d} + 6\sqrt{d}(n\delta_1^d)^{1/4}\sigma\right)^2  > \|f_{lin} - f^*\|^2_{L^2(\mu)} \\
&\qquad \qquad \geq   \frac{n\rho_{min}}{12 \rho_{max}}   \max\left(\max_i \|f_{lin} - f^*|_{\Box_i}\|^2_{C^0},\ \max_i \delta_1^2 \|f_{lin} - f^*|_{\Box_i}\|^2_{C^1}\right).
\end{align*}
Rearranging, we see that the statement of the lemma holds.
\end{proof}
Let us summarize the situation we are in in slightly more general terms. We are working in a Hilbert space $\mathcal{H}$ with norm $|\cdot|$ and dimension $n \gg 1$. Suppose $K$ is a closed symmetric convex subset of $\mathcal H$ and $A$ is a linear subspace of dimension $\widebar d$. For $\widetilde x \in \mathcal H$, let $\text{Proj}_{\mathcal{H}} (\widetilde x, A)$  denote the projection of $\widetilde x$ onto $A$ and $\text{Proj}_{\mathcal{H}} (\widetilde x, K)$  denote the projection of $\widetilde x$ onto $K$. Abbreviate these as $\Pi_A(\widetilde x)$ and $\Pi_K(\widetilde x)$, respectively.  Let $\Delta_0 := \sup_{\widetilde x \in K} |\widetilde x - \Pi_A(\widetilde x)|.$ 

For our purposes, $\mathcal{H}$ is $L^2(X)$ and has norm $(1/\sqrt{n})\|\cdot\|_{L^2(X)}$ (abbreviated as $|\cdot|$). $A$ is $\text{aff}_{\delta_1}(X)$, and $K$ is the set of all functions $f$ such that $\|f\|_{\dot{C}^{1, 1}(X)} \leq \widetilde{M}$. Note that $\Delta_0 < d \widetilde{M}\delta_1^2$ by Taylor's Theorem.

Let  $f^* \in K$, and let $\xi$ be a multivariate normal random variable taking values in $\mathcal H$ whose density at $\widetilde \xi$
with respect to the Lebesgue measure is  \begin{align*}(2\pi \sigma^2)^{n/2}\exp(- |\widetilde \xi|^2/2\sigma^2).\end{align*}
Let $y = f^* + \xi$. We want to show that $|\Pi_K(y) - f^*|$ is negligible compared to $|f^*|$. Then, we will show that this fact, together with the bound on the $\dot{C}^{1, 1}$-seminorm, implies bounded $C^0$- and $C^1$-norm.

By the triangle inequality,

\begin{eqnarray}
|\Pi_K(y) - f^*| & < &  |f^* - \Pi_A(y)| + |\Pi_A(y) - \Pi_K(y)|.
\end{eqnarray}
By Lemma \ref{lemma15} (\ref{eq:f2}) we have $ |f^* - \Pi_A(y)| <  C_d {M^*}\delta_1^2.$
Let us examine $ |\Pi_A(y) - \Pi_K(y)|$, or rather its square.

\begin{eqnarray}
 |\Pi_A(y) - \Pi_K(y)|^2 & \leq &  |\Pi_A(y) - \Pi_A(\Pi_K(y))|^2 + |\Pi_A(\Pi_K(y))  - \Pi_K(y)|^2\\
                                 & \leq & |\Pi_A(y) - \Pi_A(\Pi_K(y))|^2 + \Delta_0^2\\
                                  & \leq & \left( - |\Pi_A(y) - y|^2 + |y - \Pi_A(\Pi_K(y))|^2\right) + \Delta_0^2\\
                                 & \leq &  \left( - |\Pi_A(y) - y|^2 + ( |y - (\Pi_K(y))| + \Delta_0)^2\right) + \Delta_0^2\\
                                  & \leq &  \left( - |\Pi_A(y) - y|^2 + ( |y - f^*| + \Delta_0)^2\right) + \Delta_0^2\\
                                & \leq &  \left( - |\Pi_A(y) - y|^2 + ( |y - \Pi_A(f^*)| + 2\Delta_0)^2\right) + \Delta_0^2\\
                                & \leq &  \left(  |\Pi_A(y) - \Pi_A(f^*)|^2 + ( 4 \Delta_0 |y - \Pi_A(f^*)| + 4\Delta_0^2)\right) + \Delta_0^2\\
\end{eqnarray}
When $n$ is sufficiently large, with probability at least $1 - \exp(n^{1/(100)})$, we see that 
\begin{align*}
\left(  |\Pi_A(y) - \Pi_A(f^*)|^2 + ( 4 \Delta_0 |\xi - \Pi_A(f^*)| + 5\Delta_0^2)\right)  < 2 \frac{\widebar{d} \sigma^2}{n} + 5\Delta_0 (\sigma + \Delta_0).
\end{align*}
This can be further bounded above by \begin{align*} \sigma \left( 2 \frac{\widebar{d} \sigma}{n}  + 6 \Delta_0\right)  < \sigma \left( 2 \frac{\delta_1^{-d} \sigma}{n}  + 6 d\widetilde{M}\delta_1^2\right).
\end{align*} We then use the weighted A.M - G.M inequality with a choice of $\delta_1 = (\sigma/(6\widetilde{M}n))^{1/(d+2)}$ to ensure equality, to get
\begin{align*} 
\sigma \left( 2 \frac{\delta_1^{-d} \sigma}{n}  + 6 d\widetilde{M}\delta_1^2 \right) = (2 + d)\sigma \left(\left(\frac{\delta_1^{-d} \sigma}{n}\right)^2 \left(6 \widetilde{M}\delta_1^2\right)^d\right)^{1/(2+d)}= \frac{C_d \sigma (\sigma^2 \widetilde{M}^d)^{1/(2+d)}}{n^{\frac{2}{2+d}}}.
\end{align*}
 
Therefore, \begin{eqnarray}  |\Pi_K(y) - f^*| & < & d {M^*}\delta_1^2 + \frac{C_d \sigma^{1/2} (\sigma^2 \widetilde{M}^d)^{1/(4+2d)}}{n^{\frac{1}{2+d}}}\\
& < & C_d\left({M^*}(\sigma/(\widetilde{M}n))^{2/(d+2)} + \frac{ \sigma^{1/2} (\sigma^2 \widetilde{M}^d)^{1/(4+2d)}}{n^{\frac{1}{2+d}}}\right)\\
& < &  C_d\left(\frac{{M^*}(\sigma/\widetilde{M})^{1/(d+2)} + \sigma^{1/2} (\sigma^2 \widetilde{M}^d)^{1/(4+2d)}}{n^{\frac{1}{2+d}}}\right). \end{eqnarray}
Substituting $\widetilde{M} = n^{1/(2\max(d, 5))} < n^{1/(2d)},$ we see that the last expression can be bounded above by 

\begin{align*}C_d\left(\frac{{M^*}(\sigma)^{1/(d+2)} + \sigma^{1/2} (\sigma^4 n)^{1/(8+4d)}}{n^{\frac{1}{2+d}}}\right) = O(n^{-3/(8+4d)}),\end{align*} which is smaller than $O(|f^*|) = O(1)$ as desired.

Let $\widetilde h = (\Pi_K y)/\widetilde{M} $ and $\widetilde {f^*} = f^*/\widetilde{M}$. By the preceding discussion, $|\widetilde h| = O(\left|\widetilde{f^*}\right|) \leq O(n^{-1/(2d)})$. To prove the theorem, it suffices to show that if $\widetilde{h} \in L^2(X)$ satisfies $|\widetilde{h}| \leq O(n^{-1/(2d)})$ (assuming $d \geq 1$) and $\|\widetilde h\|_{\dot{C}^{1,1}(X)} \leq  1$, then \begin{align*}\max(\|\widetilde h\|_{C^0(X)}, \|\widetilde h\|_{C^1(X)}) \leq 1/2.\end{align*} 

We shall first show that $ \|\widetilde h\|_{C^1(X)} \leq 1/2$. Indeed, suppose 
\begin{align*} \|\widetilde h\|_{C^1(X)} > 1/2.\end{align*}
Then, by Kirszbraun's Theorem, there exist two points $a, b \in X$ such that \begin{align*}2|\widetilde{h}(a) - \widetilde{h}(b)| > |a - b|.\end{align*} However, by LeGruyer's 
Theorem for $\dot{C}^{1, 1}$,  $\|\widetilde h\|_{\dot{C}^{1,1}(X)} \leq  1$ implies that there exists a $1$-field $P$ on $X$ agreeing with $\widetilde h:X \rightarrow \mathbb{R}$ such that 
\begin{enumerate}
    \item[($L_0$)] $|P_a(a) - P_b(a)| \leq  (1/2)|a - b|^2$ for all $a, b \in X$.
  \end{enumerate}
    Now, 
\begin{align*}|P_a(a) - P_b(a)| + |P_b(a) - P_b(b)| \geq |P_a(a) - P_b(b)| > |a - b|/2\end{align*} implies that 
\begin{align*}|P_b(a) - P_b(b)| \geq |a-b|/2 - |a - b|^2/2.\end{align*} Either $|a - b| \leq 1/2$ or $ \|\widetilde h\|_{C^0(X)} > 1/8$.
Suppose the former. Then, \begin{align*}|\nabla P_b| \geq 1/4.\end{align*} Let $\widetilde{B}^d \subseteq B^d$ be a ball of radius $1/10$ containing $b$. Let $\widetilde n$ be the number of points in $\widetilde B^d$ satisfying $|P_b(x)| > 1/100$. Because $|\nabla P_b| \geq 1/4$, there is at least a small ball where this holds regardless of the value of $\widetilde h(b)$. The VC dimension $d_{vc}$ of the space of indicators of $d-$dimensional balls in $\mathbb{R}^d$ is known to be less than or equal to $d+2$. The VC inequality states that $\mathbb{P}[|\widetilde n - \E \widetilde n| < \varepsilon] > 1 - 8 \sum_{k = 0}^{d_{vc}} {{n}\choose{k}}e^{-n \varepsilon^2/32}$. The expected value of $\widetilde n$ is $C_d n$, and $\sum_{k = 0}^{d_{vc}} {{n}\choose{k}} \leq (ne/d_{vc})^{d_{vc}}$; therefore, for large enough $n$ and $\varepsilon$ small with respect to $\E \widetilde n$, $\mathbb{P}[\widetilde{n} > C_d n] \geq 1 - C_d n^{d_{vc}}e^{-C_d n} \geq 1 - \exp(-n^{1/100}).$ This implies that $|\widetilde{h}| > C_d$, which contradicts $|\widetilde{h}| \leq O(n^{-1/(2d)})$.

Now, suppose \begin{align*} \|\widetilde h\|_{C^0(X)} > 1/8.\end{align*} Let $b$ be a point where $|\widetilde h(b)| > 1/8$; without loss of generality, assume $\widetilde h(b) > 1/8$. Let $\widetilde{B}^d \subseteq B^d$ be defined as before as  a ball of radius $1/10$ containing $b$. Let $\widetilde n$ be the number of points in $\widetilde B^d$ satisfying $|P_b(x)| > 1/10$. Because  $\{x:|P_b(x)| > 1/10\}$ is a union of two halfspaces $H_+ = \{x:P_b(x) > 1/10\}$ and $H_- = \{x:P_b(x) < - 1/10\}$ whose distance from each other is $1/(5\|\nabla P_b\|)$, and such that the distance of $\partial H_+$ from $b$ is $1/(40 \|\nabla P_b\|).$ Either $1/(5\|\nabla P_b\|) < 1/20$, and consequently, $\text{Vol}_d(\widetilde{B}^d\cap\{H_+ \cup H_-\}) > C_d$ or $1/(40 \|\nabla P_b\|) \geq 1/160,$ which implies that $\text{Vol}_d(\widetilde{B}^d\cap\{H_+ \cup H_-\}) > C_d$. In either case by the bound on the VC dimensions of sublevel sets of quadratic functions (of $2d+1$), it follows by the VC inequality that $\mathbb{P}[\widetilde{n} > C_d n] \geq 1 - \exp(-n^{1/100}).$ This implies that $|\widetilde{h}| > C_d$, which contradicts $|\widetilde{h}| \leq O(n^{-1/(2d)})$.

\end{proof}

\begin{thm}\label{c11dotmain2}
Let $\|\widetilde h\|_{\dot{C}^{1, 1}(X)} \leq 1$ and  $\max(\|\widetilde h\|_{{C}^{1}(X)}, \|\widetilde h\|_{{C}^{0}(X)}) \leq 1/2.$
Then, with probability at least $ 1 - \exp(-n^{1/100})$ any minimal $\dot{C}^{1, 1}(B^d)$-norm extension $f$ of $\widetilde{h}$ to the unit ball satisfies 
\begin{align*}\max\left(\|f\|_{{C}^{1}(B^d)},\, \|f\|_{{C}^{0}(B^d)}\right) \leq 1.\end{align*}
\end{thm}
\begin{proof}
Consider a point $x \in B^d \cup \partial B^d$ where $|\nabla f|$ attains its supremum. Choose two points $a, b \in X$. By Taylor's Theorem, we have
\begin{align*}
f(a) = f(x) + \nabla f(x)^\top(a - x) + R_a
\end{align*}
and
\begin{align*}
f(b) = f(x) + \nabla f(x)^\top(b - x) + R_b,
\end{align*}
where $R_a := \frac{1}{2} \int_{0}^{1} (\nabla f(x + t(a - x)) - \nabla f(x))^\top (a - x) dt$ and $R_b$ is defined analogously. Subtracting the first equation from the second, dividing by $\|b - a\|$, and taking absolute values yields
\begin{align*}
\frac{|\nabla f(x)^\top(b - a)|}{\|b - a\|} &= \frac{|f(b) - f(a) + R_a - R_b|}{\|b - a\|}\\
&\leq \frac{|f(b) - f(a)|}{\|b - a\|} + \frac{|R_a - R_b|}{\|b - a\|}.
\end{align*}
Since $\|\widetilde h\|_{C^1(X)} \leq \frac{1}{2}$ and $f$ agrees with $\widetilde h$ on $X$, Kirszbraun's Theorem implies that $\frac{|f(b) - f(a)|}{\|b - a\|} \leq \frac{1}{2}$. To bound the second term, use Taylor's Theorem to write
\begin{align*}
f(b) - f(a) &=  \nabla f(a)^\top(b - a) + \frac{1}{2} \int_{0}^{1} (\nabla f(a + t(b - a)) - \nabla f(a))^\top (b - a) dt.
\end{align*}
Combining this with the previous applications of Taylor's Theorem, we have
\begin{align*}
|R_b - R_a| &\leq \left|(\nabla f(a) - \nabla f(x))^\top(b - a)\right| + \left|\frac{1}{2} \int_{0}^{1} (\nabla f(a + t(b - a)) - \nabla f(a))^\top (b - a) dt\right|\\
&\leq \|(\nabla f(a) - \nabla f(x))^\top\|\|(b - a)\| + \frac{1}{2}\|f\|_{\dot{C}^{1, 1}}\|b - a\|^2\\
&\leq \|f\|_{\dot{C}^{1, 1}}\|a - x\|\|(b - a)\| + \frac{1}{2}\|f\|_{\dot{C}^{1, 1}}\|b - a\|^2\\
&\leq \|(b - a)\|\left(\|a - x\| + \frac{1}{2}\|b - a\|\right),
\end{align*}
implying that
\begin{align*}
\frac{|\nabla f(x)^\top(b - a)|}{\|b - a\|} &\leq \frac{1}{2} + \|a - x\| + \frac{1}{2}\|b - a\|.
\end{align*}
${|\nabla f(x)^\top(b - a)|}\big/{\|b - a\|}$ should only depend on the unit vector $({b - a})\big/{\|b - a\|}$, and it is maximized when $b - a$ is equal to $\nabla f(x)$. If $n$ is large enough, there are points $a$ and $b$ arbitrarily close to $x$ such that $\nabla f(x)$ is approximated with any desired precision by $b - a$. Let $\widetilde n$ be the number of points in $X_0 := X \cap B^d(x, \varepsilon')$. The VC dimension $d_{vc}$ of the space of indicators of $d-$dimensional balls in $\mathbb{R}^d$ is a function of $d$. The VC inequality states that $\mathbb{P}[\widetilde{n} > C_d n] > 1 - 8 \sum_{k = 0}^{d_{vc}} {{n}\choose{k}}e^{-n \varepsilon^2/32} \geq 1 - \exp(-n^{1/100})$. Now choose one of the points in $X_0$ and call it $b$. For each of the other points $b_i \in X_0$, form the unit vector $c_i := (b - b_i)/\|b - b_i\|$. By VC theory, the probability that at least one of the $c_i$ is within $\varepsilon''$ of $\nabla f(x)/\|\nabla f(x)\|$ is greater than $1 - \exp(-n^{1/100})$. Therefore, with probability at least $1 - 2 \exp(-n^{1/100}) $, we can choose points $a$ and $b$ such that
\begin{align*}
\sup_{x \in B^d} |\nabla f(x)| &= \left|{\nabla f(x)^\top}\left(\frac{\nabla f(x)}{\|\nabla f(x)\|}\right)\right|\\
 &\leq \left|{\nabla f(x)^\top}\left(\frac{b - a}{\|b - a\|}\right)\right| + \left|{\nabla f(x)^\top}\left(\frac{\nabla f(x)}{\|\nabla f(x)\|} - \frac{b - a}{\|b - a\|}\right)\right|\\
 &\leq \frac{1}{2} + \|a - x\| + \frac{1}{2}\|b - a\| + \sup_{x \in B^d} |\nabla f(x)| \left\|\frac{\nabla f(x)}{\|\nabla f(x)\|} - \frac{b - a}{\|b - a\|}\right\|\\
 &\leq \frac{1}{2} + \frac{3}{2} \varepsilon' +  \sup_{x \in B^d} |\nabla f(x)| \varepsilon''.
\end{align*}
Thus,
\begin{align*}
\sup_{x \in B^d} |\nabla f(x)| \leq \frac{\frac{1}{2} + \frac{3}{2} \varepsilon' }{1 - \varepsilon''},
\end{align*}
which is enough to show that $\|f\|_{{C}^{1}(B^d)} < 1$.

Now consider a point $x' \in B^d \cup \partial B^d$ where $|f|$ is maximum, and choose a nearby point $a' \in X$. By Taylor's Theorem, where $c'$ is between $a'$ and $x'$,
\begin{align*}
|f(x')| &= |f(a') + \nabla f(c')^\top (a' - x')|\\
&\leq \|\widetilde h\|_{{C}^{0}(X)} + \|f\|_{{C}^{1}(B^d)} \|a' - x'\|.
\end{align*}
The VC inequality shows that, with probability greater than $1 - \exp(-n^{1/100})$, there are at least $C_d n$ points within $\varepsilon'$ of $x$. Therefore, for large enough $n$ we can choose $a'$ arbitrarily close to $x'$, proving that $\|f\|_{{C}^{0}(B^d)} < 1$.
\end{proof}

\subsection{Wells' Construction for \boldmath{$\widehat{f}$} Given \boldmath{$\Lip(\nabla \widehat{f})$}} \label{sec:wells}
Given $M = \Lip(\nabla \widehat{f})$ and the estimates $\widehat{f}(a)$ and $\widehat{D_a f}$ for all $a \in E$, it remains to construct the interpolant $\widehat{f} \in \dot{C}^{1, 1}(\R^d)$.  We may now apply solution methods from the noiseless function interpolation problem. We summarize the solution provided by \citet*{wells1973differentiable} here.

Wells' construction takes $E \subset \R^d$, the $1$-field $P: E \rightarrow \mP$ consisting of function values $\set{f(a)}_{a \in E}$ and gradients $\set{D_a f}_{a \in E}$, and a value $M = \Lip(\nabla \widehat{f})$ as inputs.  A necessary condition for Wells' construction to hold is that 
\begin{equation}\label{eq:wells_condition}
f(b) \leq f(a) + \frac{1}{2}(D_a f + D_b f) \cdot (b-a) + \frac{M}{4}|b-a|^2 - \frac{1}{4M}|D_a f - D_b f|^2, \quad \forall a, b \in E, 
\end{equation}
for which the optimal objective function value and gradients returned by the methods in Sections \ref{sec:legruyer} and \ref{sec:vaidya} satisfy.  

For all $a \in E$, Wells defines the shifted points
$$
\widetilde{a} = a - \frac{D_a f}{M},
$$
and associates a type of distance function for any $x \in \R^d$ to that point,
$$
d_a(x) = f(a) - \frac{1}{2M}|D_a f|^2 + \frac{M}{4}|x - \widetilde{a}|^2.
$$
Using the shifted points, every subset $S \subset E$ is associated with several new sets:
\begin{align*}
\widetilde{S} &= \set{\widetilde{a} \mid a \in S}, \\
S_H &= \text{the smallest affine space containing } \widetilde{S}, \\
\widehat{S} &= \text{the convex hull of } \widetilde{S}, \\
S_E &= \set{x \in \R^d \mid d_a(x) = d_b(x) \text{ for all } a, b \in S}, \\
S_\ast &= \set{x \in \R^d \mid d_a(x) = d_b(x) \leq d_c(x) \text{ for all } a, b \in S,\; c \in E}, \\
S_C &= S_H \cap S_E.
\end{align*}
Note that $S_H \perp S_E$, so $S_C$ is a singleton.  Wells next defines the collection of subsets
$$
\mK = \set{S \subset E \mid \exists x \in S_\ast \text{ such that } d_S(x) < d_{E\setminus S}(x)},
$$
and a new collection of sets $\set{T_S}_{S \in \mK}$, where 
\begin{equation}\label{eq:wells_cover}
T_S = \frac{1}{2}(\widehat{S} + S_\ast) = \set{\frac{1}{2}(y + z) \mid y \in \widehat{S}, \; z \in S_\ast}, \quad S \in \mK.  
\end{equation}
The collection $\set{T_S}_{S \in \mK}$ form a partition of $\R^d$ in the sense that overlapping sets have Lebesgue measure 0.  On each set $T_S$, Wells defines a function $\widehat{f}_S: T_S \rightarrow \R$ which is a local piece of the interpolating function $\widehat{f}$:
\begin{equation} \label{eq:wells_func}
\widehat{f}_S(x) = d_S(S_C) + \frac{M}{2} \text{dist}(x, S_H)^2 - \frac{M}{2}\text{dist}(x, S_E)^2, \quad x \in T_S, \; S \in \mK, 
\end{equation}
where as usual for sets $A, B \subset \R^d$, we have $\text{dist}(A, B) = \inf_{x \in A, \; y \in B} |x - y|$.  The final function $\widehat{f}: \R^d \rightarrow \R$ is then defined using \eqref{eq:wells_cover} and \eqref{eq:wells_func}:
\begin{equation} \label{eq:wells}
\widehat{f}(x) = \widehat{f}_S(x), \quad \text{if } x \in T_S.
\end{equation}
The gradient of $\widehat{f}_S$ is 
$$
\nabla \widehat{f}_S(x) = \frac{M}{2}(z - y), \quad \text{where } x = \frac{1}{2}(y + z), \; y \in \widehat{S}, z \in S_\ast.
$$
The function $\widehat{f}: \R^d \rightarrow \R$ of \eqref{eq:wells} satisfies the following:
\begin{thm}[Wells' Construction]
Given a finite set $E \subset \R^d$, a $1$-field \newline $P: E \rightarrow \mP$, and a constant $M$ satisfying \eqref{eq:wells_condition}, the function $\widehat{f}: \R^d \rightarrow \R$ defined by \eqref{eq:wells_func} is in $\dot{C}^{1, 1}(\R^d)$ and satisfies
\begin{enumerate}
\item $J_a \widehat{f} = P_a$ for all $a \in E$.
\item $\Lip(\nabla \widehat{f}) = M$.
\end{enumerate}
\end{thm}
\cite{herbert2014computing} provide efficient algorithms to implement Wells' construction.  We refer the reader to \cite{herbert2014computing} for the details, but state briefly the computational cost of their methods.  Let $m = |\mK|$, and note that in the worst case $m = O(n^{\lceil d/2 \rceil})$.  As stated by \cite{herbert2014computing}, a pessimistic bound on the storage as well as the number of computations for the one time work is $O(m^2)$. The query work is then also bounded by $O(m^2)$, however using more efficient querying algorithms to find the set $T_S \in \mK$ to which a given $x \in \R^d$ belongs can lessen the work significantly.  Using a tree structure, for example, will require $O(\log m) = O(\log n)$ work per query if the tree is balanced, but this need not be the case in general.

\section{Simulation} \label{sec:sim}
\FloatBarrier
We numerically compute the empirical risk minimizer $\widehat{f}$ over the functional class $\mathcal{F}_M = \{ f \mid \| f \|_{\dot{C}^{1,1}(B^d)} 
\leq M\}$. To do so, we solve the optimization problem \eqref{eq:alg} for a $1$-field $P$ over the finite set $E$. This can be done efficiently with the algorithm described in Section \ref{sec:vaidya}, or using any constrained, convex optimization algorithm (such as interior point methods). The $1$-field $P$ is extended to a function in $\mathcal{F}_{M}$, which is $\widehat{f}$, using the algorithm described by \cite{herbert2014computing}. 

The underlying function $f: \R^2 \rightarrow \R$ is taken as:
\begin{equation*}
\forall \, x = (x_1, x_2) \in \R^2, \quad f(x_1, x_2) = \left\{
\begin{array}{ll}
\cos (\pi x_1) \sin (\pi x_2) \exp \left( -\frac{1}{1-|x|^2} \right), & |x| < 1, \\
0, & |x| \geq 1,
\end{array}
\right.
\end{equation*}
which is supported in the unit ball $B^2$. The training points $E$ are sampled uniformly from $B^2$, and noisy function values $y(a) = f(a) + \xi(a)$ for $a \in E$ are recorded, where $\xi(a)$ is i.i.d. Gaussian white noise with standard deviation $\sigma$. To approximate the generalization error between $\widehat{f}$ and $f$ on $B^2$, the unit cube in which $B^2$ is inscribed is sampled on a grid containing $2^{14}$ points ($2^7$ along each axis). As such, all errors over $B^2$ described below are numerically approximated on the intersection of this grid with $B^2$. 

The error $\sup_{x \in B^2} | f(x) - \widehat{f}(x)|$ between $f$ and the empirical risk minimizer is plotted in Figure \ref{fig: max err on B} as a function of $n$, for various values of the noise standard deviation $\sigma$. The value of $M$ grows with $n$ according to $M = O(1/n^{10})$, as in the proof of Theorem \ref{thm:maintheoremERM}. Figure \ref{fig: rmse on B} plots the generalization error for the quadratic loss, i.e., $\left( \int_{B^2} |f(x) - \widehat{f}(x)|^2 \, dx \right)^{1/2}$. Both figures show that the generalization error generally decreases as $n$ (and correspondingly $M$) increase. 

Figure \ref{fig:interpolants} gives a qualitative assessment of the empirical risk minimizer $\widehat{f}$ in the noiseless setting ($\sigma = 0$), by plotting the original function $f$ and several versions of $\widehat{f}$ for selected values of $n$. As expected, the empirical risk minimizer $\widehat{f}$ visually appears to better match the underlying function $f$ as $n$ increases. Figure \ref{fig: interpolants 2} fixes $n = 84$ and plots $\widehat{f}$ for increasing values of the noise standard deviation. While for large noise ($\sigma = 0.5$) the empirical risk minimizer $\widehat{f}$ deviates noticeably from $f$, for lower noise values ($\sigma \leq 0.25$), $\widehat{f}$ is a stable approximation of $f$. 

\begin{figure}[htbp]
\center
\subfigure[Maximum error]{
\includegraphics[width=.45\textwidth]{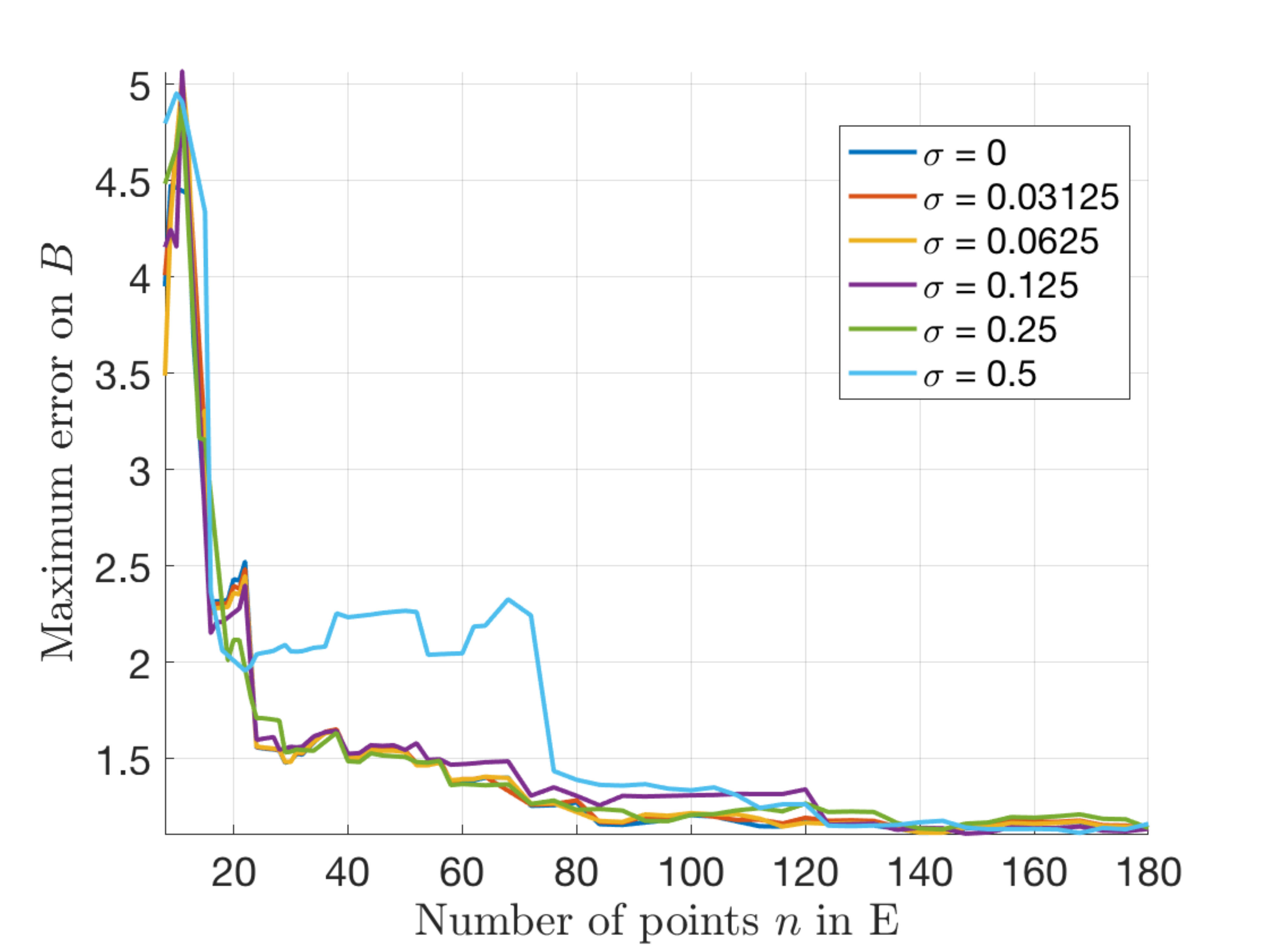}
\label{fig: max err on B}
}
\subfigure[Root mean square error]{
\includegraphics[width=.45\textwidth]{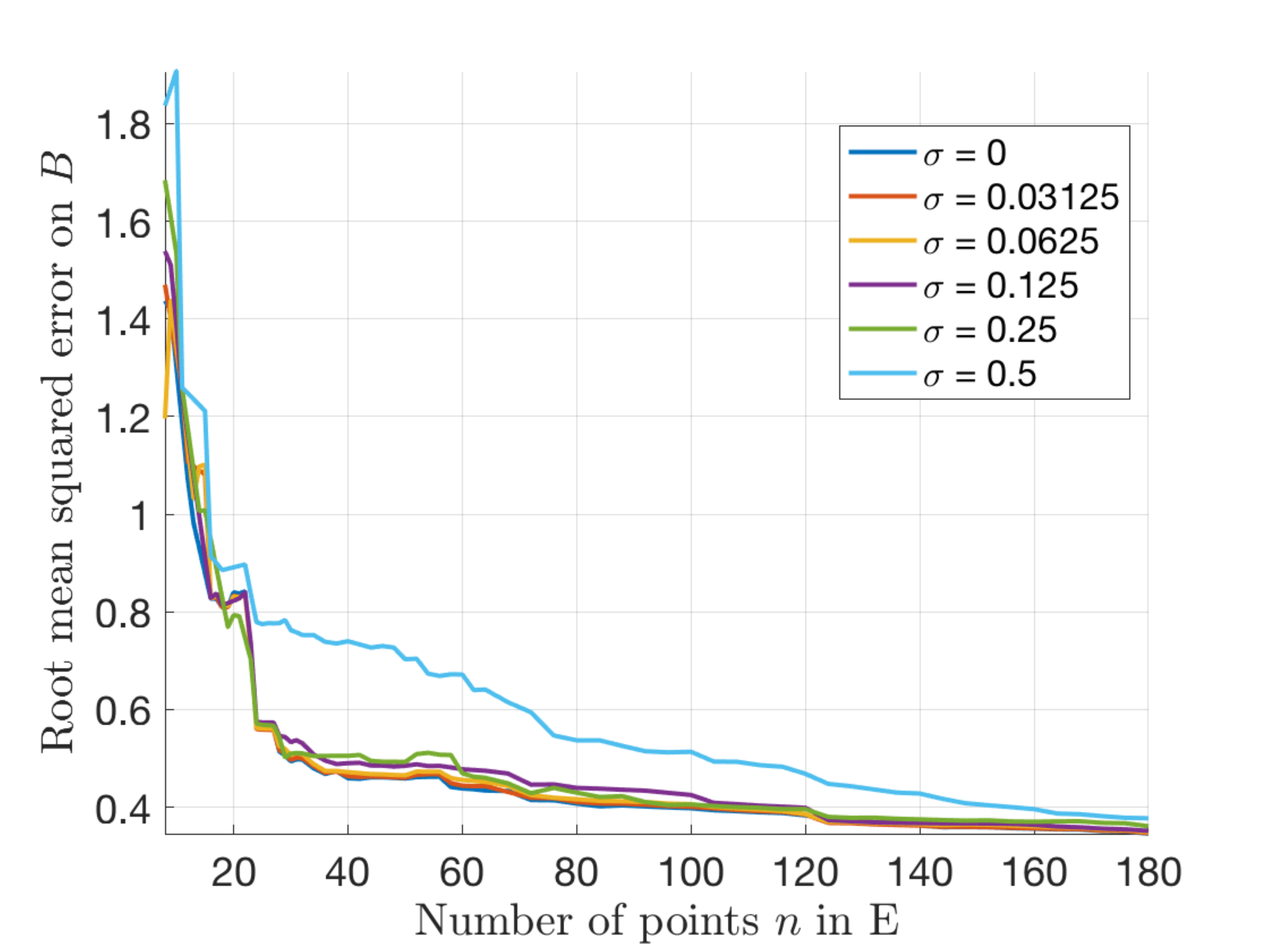}
\label{fig: rmse on B}
}
\caption{Generalization error as a function of $|E| = n$}
\end{figure}

\begin{figure}[htbp]
\centerline{
\includegraphics[width=.9\textwidth]{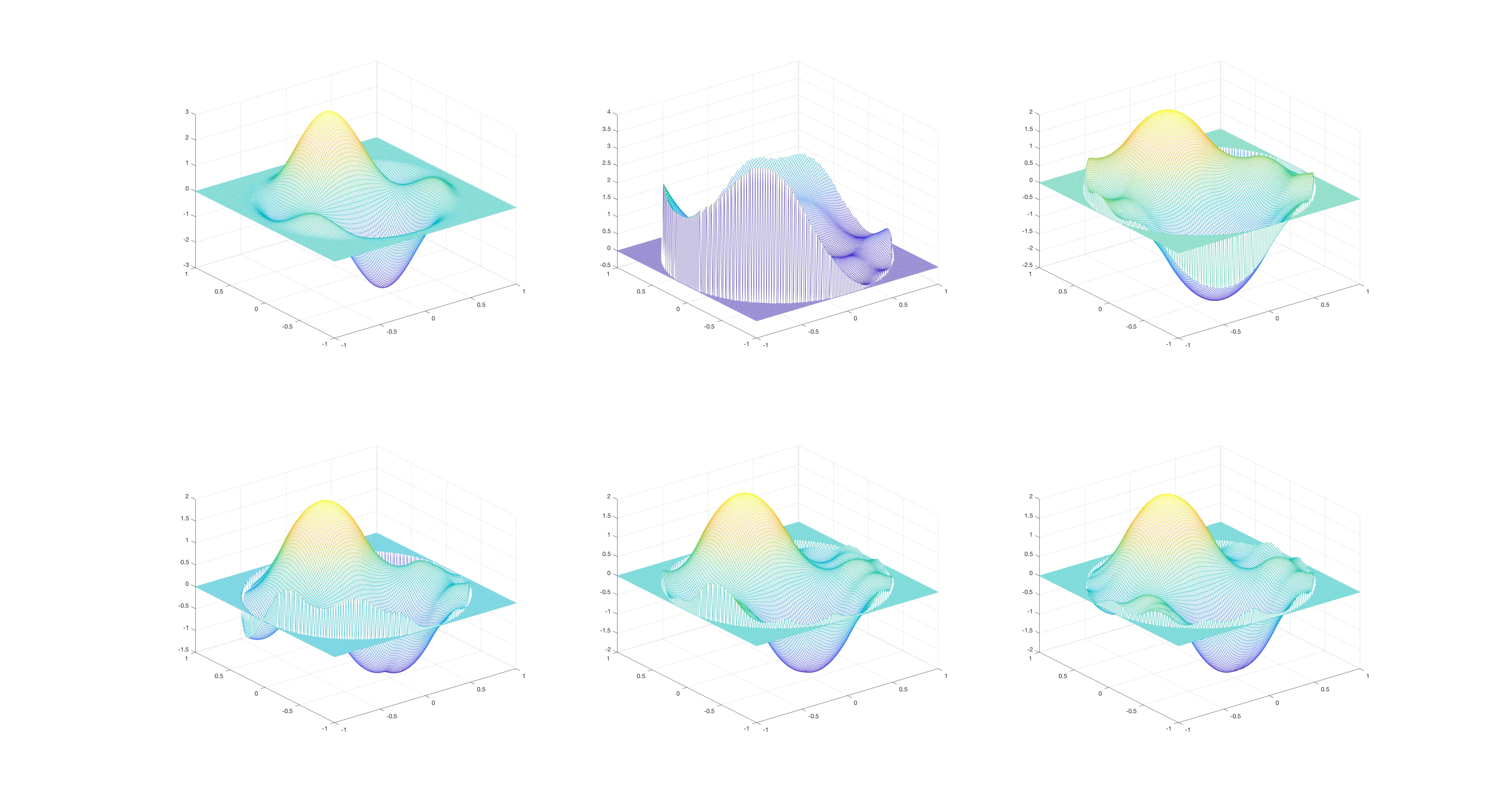}
} \vspace{-8pt}
\caption{Upper left: The original function $f$. Subsequent plots, moving left  to right and then to the second row: the empirical risk minimizer $\widehat{f}$ computed with $n = 8, 23, 38, 84, 180$ samples, respectively, and for $\sigma = 0$.}
\label{fig:interpolants}
\end{figure}
\begin{figure}[htbp]
\centerline{
\includegraphics[width=.9\textwidth]{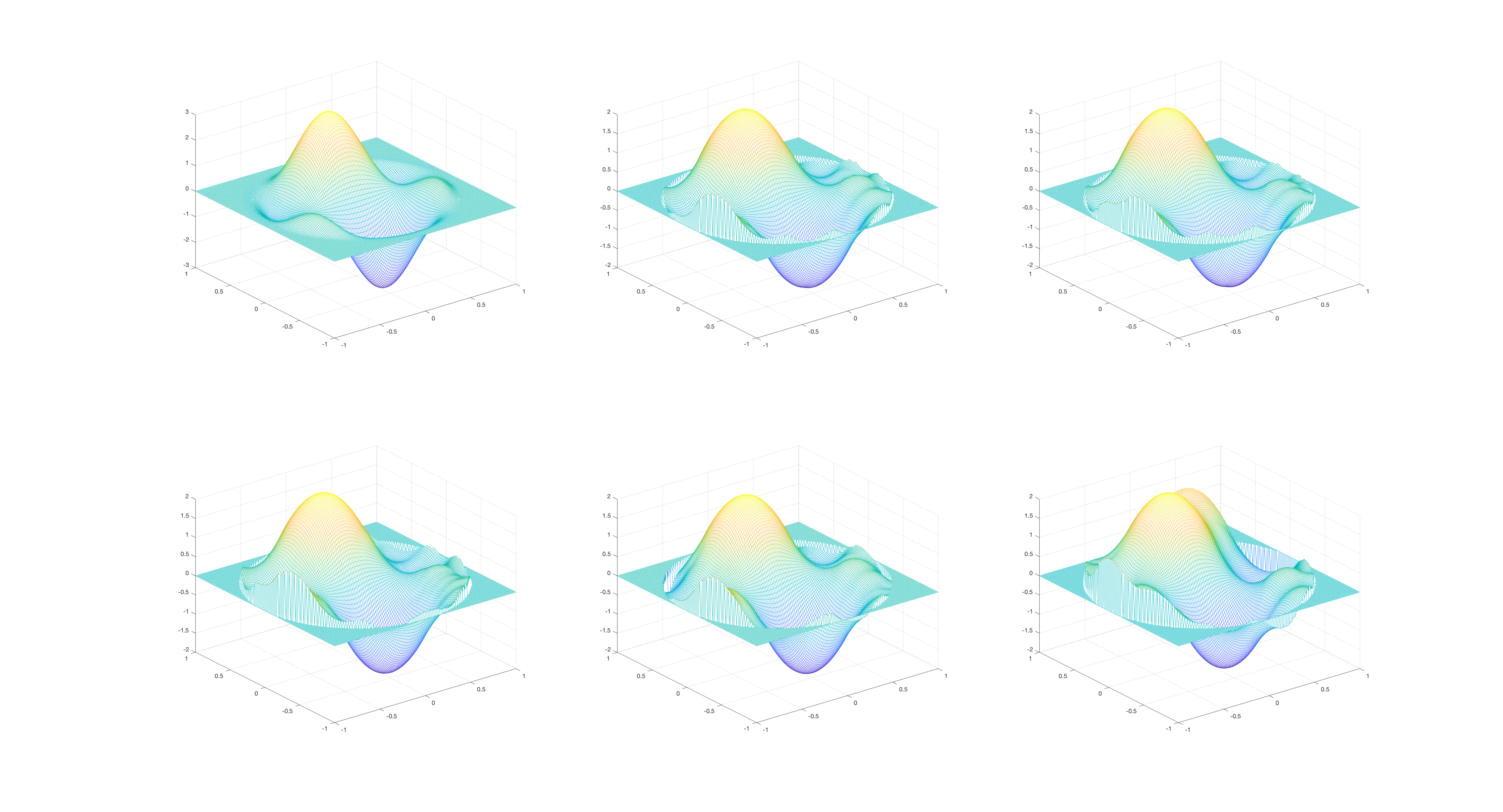}
} \vspace{-8pt}
\caption{Upper left: The original function $f$. Subsequent plots, moving left  to right and then to the second row: the empirical risk minimizer $\widehat{f}$ computed from noisy data with $\sigma = 2^{-j}$ for $j=5, 4, 3, 2, 1$, respectively, and for $n = 84$ samples.}
\label{fig: interpolants 2}
\end{figure}
\FloatBarrier

\section{Discussion}
In this paper, we extend the function interpolation problem considered by \cite{herbert2014computing} to the regression setting, where function values $f(a)$ are observed with uncertainty over finite $a\in E$. We impose smoothness on the approximating function by considering regression solutions in the class of $C^{1,1}(\R^d)$ functions. Minimizing the risk over this function class is computationally tractable optimization problem, requiring $O( (d+1)^2 n^2)$ calls to a separation oracle using Vaidya's algorithm. We present a separating hyperplane that requires $O(n^2)$ operations, and given the output of Vaidya's algorithm, reconstruct the interpolant using efficient implementations of Wells' construction proposed by \cite{herbert2014computing}.  

We derive uniform bounds relating the empirical risk of the regression solution to the true risk using empirical processes methods. The covering number of the class of $C^{1,1}(\R^d)$ functions is known and can be used to derive the covering number of Lipschitz loss classes. Our loss class is unbounded, but by conditioning on a suitable bound that increases with $n$, we obtain high probability bounds. As a consequence of the uniform risk bounds, almost sure convergence of the empirical risk minimizer is also guaranteed. These theoretical contributions are supported by numerical results via simulation.


\section*{Acknowledgments}
We would like to thank Charles Fefferman and Nahum Zobin for the opportunity to present some of these results at the 2017 Whitney Problems Workshop at the College of William and Mary, as well as Erwan Le Gruyer for helpful comments. 

Matthew Hirn is supported by the Alfred P. Sloan Fellowship (grant FG-2016-6607), the DARPA YFA (grant D16AP00117), and NSF grant 1620216.  Hari Narayanan is supported by a Ramanujan Fellowship.  Jason Xu is supported by NSF MSPRF DMS-1606177. NSF grant 1620102 partially supported Adam Gustafson, Kitty Mohammed, and Hari Narayanan. 

\bibliography{c11_refs}

\end{document}